\def\year{2018}\relax
\documentclass[letterpaper]{article} 
\usepackage{aaai18}  
\usepackage{times}  
\usepackage{helvet}  
\usepackage{courier}  
\usepackage{url}  
\usepackage{graphicx}  
\usepackage{color}
\usepackage{dsfont}
\usepackage{multicol}
\usepackage{booktabs}
\usepackage{flushend}
\usepackage{algorithm}
\usepackage{algorithmic}
\usepackage{amsmath, amssymb, amsthm, mathtools}

\newtheorem{lemma}{Lemma}
\newtheorem{theorem}{Theorem}
\newtheorem{assumption}{Assumption}
\newtheorem{proposition}{Proposition}

\newtheorem{lemmaA}{Lemma A\ignorespaces}
\newtheorem{duplicateLemma}{Lemma}
\newtheorem{duplicateTheorem}{Theorem}
\newtheorem{duplicateProposition}{Proposition}

\DeclarePairedDelimiter\ceil{\lceil}{\rceil}
\DeclarePairedDelimiter\floor{\lfloor}{\rfloor}

\begin{document}
%
\title{Budget-Constrained Multi-Armed Bandits with Multiple Plays}
\author{Datong P. Zhou\textsuperscript{1} \and Claire J. Tomlin\textsuperscript{2}\\
\textsuperscript{1}{Dept. of Mechanical Engineering},
\textsuperscript{2}{Dept. of Electrical Engineering and Computer Sciences}\\
University of California, Berkeley, CA 94720 \\ \{datong.zhou, tomlin\}@berkeley.edu}
\maketitle
\vspace*{-1.0cm}
\begin{abstract}
We study the multi-armed bandit problem with multiple plays and a budget constraint for both the stochastic and the adversarial setting. At each round, exactly $K$ out of $N$ possible arms have to be played (with $1\leq K \leq N$). In addition to observing the individual rewards for each arm played, the player also learns a vector of costs which has to be covered with an a-priori defined budget $B$. The game ends when the sum of current costs associated with the played arms exceeds the remaining budget.

Firstly, we analyze this setting for the stochastic case, for which we assume each arm to have an underlying cost and reward distribution with support $[c_{\min}, 1]$ and $[0, 1]$, respectively. We derive an Upper Confidence Bound (UCB) algorithm which achieves $O(NK^4 \log B)$ regret.

Secondly, for the adversarial case in which the entire sequence of rewards and costs is fixed in advance, we derive an upper bound on the regret of order $O(\sqrt{NB\log(N/K)})$ utilizing an extension of the well-known \texttt{Exp3} algorithm. We also provide upper bounds that hold with high probability and a lower bound of order $\Omega((1 - K/N)^2 \sqrt{NB/K})$. 
\end{abstract}

%
%

\section{Introduction}
The multi-armed bandit (MAB) problem has been extensively studied in machine learning and statistics as a means to model online sequential decision making. In the classic setting popularized by \cite{Auer:2002ab}, \cite{Auer:2002aa}, the decision-maker selects exactly one arm at a given round $t$, given the observations of realized rewards from arms played in previous rounds $1, \ldots, t-1$. The goal is to maximize the cumulative reward over a fixed horizon $T$, or equivalently, to minimize regret, which is defined as the difference between the cumulative gain achieved, had the decision-maker always played the best arm, and the realized cumulative gain. The analysis of this setting reflects the fundamental tradeoff between the desire to learn better arms (exploration) and the possibility to play arms believed to have high payoff (exploitation).

A variety of practical applications of the MAB problem include placement of online advertising to maximize the click-through rate, in particular online sponsored search auctions \cite{Rusmevichientong:2005aa} and ad-exchange platforms \cite{Chakraborty:2010aa}, channel selection in radio networks \cite{Huang:2008aa}, or learning to rank web documents \cite{Radlinski:2008aa}. As acknowledged by \cite{Ding:2013aa}, taking an action (playing an arm) in practice is inherently costly, yet the vast majority of existing bandit-related work used to analyze such examples forgoes any notion of cost. Furthermore, the above-mentioned applications rarely proceed in a strictly sequential way. A more realistic scenario is a setting in which, at each round, \textit{multiple} actions are taken among the set of all possible choices.

These two shortcomings motivate the theme of this paper, as we investigate the MAB problem under a budget constraint in a setting with time-varying rewards and costs and multiple plays. More precisely, given an a-priori defined budget $B$, at each round the decision maker selects a combination of $K$ distinct arms from $N$ available arms and observes the individual costs and rewards, which corresponds to the \textit{semi-bandit} setting. The player pays for the materialized costs until the remaining budget is exhausted, at which point the algorithm terminates and the cumulative reward is compared to the theoretical optimum and defines the weak regret, which is the expected difference between the payout under the best fixed choice of arms for all rounds and the actual gain. In this paper, we investigate both the stochastic and the adversarial case. For the stochastic case, we derive an upper bound on the expected regret of order $O(NK^4 \log B)$, utilizing Algorithm \texttt{UCB-MB} inspired by the upper confidence bound algorithm \texttt{UCB1} first introduced by \cite{Auer:2002ab}. For the adversarial case, Algorithm \texttt{Exp3.M.B} upper and lower-bounds the regret with $O(\sqrt{NB\log(N/K)})$ and $\Omega((1-K/N)^2 \sqrt{NB/K})$, respectively. These findings extend existing results from \cite{Uchiya:2010aa} and \cite{Auer:2002aa}, as we also provide an upper bound that holds with high probability. To the best of our knowledge, this is the first case that addresses the \textit{adversarial} budget-constrained case, which we therefore consider to be the main contribution of this paper.

\subsection*{Related Work}
In the extant literature, attempts to make sense of a cost component in MAB problems occur in \cite{Tran-Thanh:2010aa} and \cite{Tran-Thanh:2012aa}, who assume \textit{time-invariant} costs and cast the setting as a knapsack problem with only the rewards being stochastic. In contrast, \cite{Ding:2013aa} proposed algorithm \texttt{UCB-BV}, where per-round costs and rewards are sampled in an IID fashion from unknown distributions to derive an upper bound on the regret of order $O(\log B)$. The papers that are closest to our setting are \cite{Badanidiyuru:2013aa} and \cite{Xia:2016aa}. The former investigates the stochastic case with a resource consumption. Unlike our case, however, the authors allow for the existence of a ``null arm'', which is tantamount to skipping rounds, and obtain an upper bound of order $O(\sqrt{B})$ rather than $O(\log B)$ compared to our case. The latter paper focuses on the stochastic case, but does not address the adversarial setting at all.


The extension of the single play to the multiple plays case, where at each round $K\geq 1$ arms have to be played, was introduced in \cite{Anantharam:1986aa} and \cite{Agrawal:1990aa}. However, their analysis is based on the original bandit formulation introduced by \cite{Lai:1985aa}, where the regret bounds only hold asymptotically (in particular not for a finite time), rely on hard-to-compute index policies, and are distribution-dependent. Influenced by the works of \cite{Auer:2002ab} and \cite{Agrawal:2002aa}, who popularized the usage of easy-to-compute upper confidence bounds (UCB), a recent line of work has further investigated the combinatorial bandit setting. For example, \cite{Gai:2012aa} derived an $O(NK^4 \log T)$ regret bound in the stochastic semi-bandit setting, utilizing a policy they termed ``Learning with Linear Rewards'' (LLR). Similarly, \cite{Chen:2013aa} utilize a framework where the decision-maker queries an oracle that returns a fraction of the optimal reward. Other, less relevant settings to this paper are found in \cite{Cesa-Bianchi:2009aa} and later \cite{Combes:2015aa}, who consider the adversarial bandit setting, where only the sum of losses for the selected arms can be observed. Furthermore, \cite{Kale:2010aa} investigate bandit slate problems to take into account the ordering of the arms selected at each round. Lastly, \cite{Komiyama:2015aa} utilize Thompson Sampling to model the stochastic MAB problem.


%
%

\section{Main Results}
In this section, we formally define the budgeted, multiple play multi-armed bandit setup and present the main theorems, whose results are provided in Table \ref{tab:regret_bounds_overview} together with a comparison to existing results in the literature. We first describe the stochastic setting (Section \ref{sec:stochastic_main}) and then proceed to the adversarial one (Section \ref{sec:adversarial_main}). Illuminating proofs for the theorems in this section are presented in Section \ref{sec:proofs}. Technical proofs are relegated to the supplementary document.

\begin{table*}[hbtp]
\centering
\begin{tabular}{*4c}
\toprule
Algorithm & Upper Bound & Lower Bound & Authors \\
\hline
\texttt{Exp3} & $O\left(\sqrt{NT\log N}\right)$ & $\Omega
\left(\sqrt{NT}\right)$ & \cite{Auer:2002aa}  \\
\texttt{Exp3.M} & $O\left(\sqrt{NTK\log \frac{N}{K}}\right)$ & $\Omega\left(\left(1 - \frac{K}{N}\right)^2\sqrt{NT}\right)$ & \cite{Uchiya:2010aa}  \\
\texttt{Exp3.M.B} & $O\left(\sqrt{NB\log \frac{N}{K}}\right)$ & $\Omega\left(\left(1 - \frac{K}{N}\right)^2\sqrt{NB/K}\right)$ & This paper  \\
\hline
\texttt{Exp3.P} & \multicolumn{2}{c}{$O\left( \sqrt{NT\log\left(NT/\delta\right)}  + \log(NT/\delta)\right)$} & \cite{Auer:2002aa} \\
\texttt{Exp3.P.M} & \multicolumn{2}{c}{$O\left( K^2 \sqrt{NT\frac{N-K}{N-1}\log\left(NT/\delta\right)}  + \frac{N-K}{N-1} \log(NT/\delta)\right)$} & This paper \\
\texttt{Exp3.P.M.B} & \multicolumn{2}{c}{$O\left( K^2 \sqrt{\frac{NB}{K}\frac{N-K}{N-1}\log\left(\frac{NB}{K\delta}\right)}  + \frac{N-K}{N-1} \log\left(\frac{NB}{K\delta}\right)\right)$} & This paper \\
\hline
\texttt{UCB1} & $O(N\log T)$ & & \cite{Auer:2002ab} \\
\texttt{LLR} & $O(NK^4 \log T)$ & & \cite{Gai:2012aa} \\
\texttt{UCB-BV} & $O(N\log B)$ & & \cite{Ding:2013aa} \\
\texttt{UCB-MB} & $O(NK^4\log B)$ & & This paper \\
\bottomrule
\end{tabular}
\vspace{0.1cm}
\caption{Regret Bounds in Adversarial and Stochastic Bandit Settings}
\label{tab:regret_bounds_overview}
\end{table*}

\subsection{Stochastic Setting}\label{sec:stochastic_main}
The definition of the stochastic setting is based on the classic setup introduced in \cite{Auer:2002ab}, but is enriched by a cost component and a multiple play constraint. Specifically, given a bandit with $N$ distinct arms, each arm indexed by $i\in [N]$ is associated with an unknown reward and cost distribution with unknown means $0 < \mu_r^i \leq 1$ and $0 < c_{\min} \leq \mu_c^i \leq 1$, respectively. Realizations of costs $c_{i,t} \in [c_{\min}, 1]$ and rewards $r_{i,t} \in [0,1]$ are independently and identically distributed. At each round $t$, the decision maker plays exactly $K$ arms ($1 \leq K \leq N$) and subsequently observes the individual costs and rewards only for the played arms, which corresponds to the \textit{semi-bandit} setting. Before the game starts, the player is given a budget $0 < B \in\mathbb{R}_+$ to pay for the materialized costs $\lbrace c_{i,t}~|~ i\in a_t\rbrace$, where $a_t$ denotes the indexes of the $K$ arms played at time $t$. The game terminates as soon as the sum of costs at round $t$, namely $\sum_{j\in a_t} c_{j,t}$ exceeds the remaining budget.

Notice the minimum $c_{\min}$ on the support of the cost distributions. This assumption is not only made for practical reasons, as many applications of bandits come with a minimum cost, but also to guarantee well-defined ``bang-per-buck'' ratios $\mu^i = \mu_r^i / \mu_c^i$, which our analysis in this paper relies on.

The goal is to design a deterministic algorithm $\mathcal{A}$ such that the expected payout $\mathbb{E}\left[G_{\mathcal{A}}(B)\right]$ is maximized, given the budget and multiple play constraints. Formally:
\begin{equation}\label{eq:bandit_optimization}
\begin{aligned}
& \underset{a_1, \ldots, a_{\tau_{\mathcal{A}}(B)}}{\text{maximize}}
& & \mathbb{E}\left[ \sum_{t=1}^{\tau_{\mathcal{A}}(B)} \sum_{i\in a_t} r_{i,t} \right] \\
& \text{subject to}
& & \mathbb{E}\left[ \sum_{t=1}^{\tau_{\mathcal{A}}(B)} \sum_{i\in a_t} c_{i,t} \leq B \right] \\
& & & |a_t| = K,~ 1\leq K \leq N ~\forall~t\in [\tau_{\mathcal{A}}(B)]
\end{aligned}
\end{equation}
In \eqref{eq:bandit_optimization}, $\tau_{\mathcal{A}}(B)$ is the stopping time of algorithm $\mathcal{A}$ and indicates after how many steps the algorithm terminates, namely when the budget is exhausted. The expectation is taken over the randomness of the reward and cost distributions.

The performance of algorithm $\mathcal{A}$ is evaluated on its expected regret $\mathcal{R}_{\mathcal{A}}(B)$, which is defined as the difference between the expected payout (gain) $\mathbb{E}[G_{\mathcal{A}^\ast}]$ under the optimal strategy $\mathcal{A}^\ast$ (which in each round plays $a^\ast$, namely the set of $K$ arms with the largest bang-per-buck ratios) and the expected payout $\mathbb{E}[G_{\mathcal{A}}]$ under algorithm $\mathcal{A}$:
\begin{align}\label{eq:regret_definition}
\mathcal{R}_{\mathcal{A}}(B) = \mathbb{E}[G_{\mathcal{A}^\ast}(B)] - \mathbb{E}[G_{\mathcal{A}}(B)].
\end{align}
Our main result in Theorem \ref{thm:stochastic_case_sublinear_regret} upper bounds the regret achieved with Algorithm \ref{thm:stochastic_case_sublinear_regret}. Similar to \cite{Auer:2002ab} and \cite{Ding:2013aa}, we maintain time-varying upper confidence bounds $U_{i,t}$ for each arm $i$
\begin{align}
U_{i,t} &= \bar{\mu}_t^i + e_{i,t},
\end{align}
where $\bar{\mu}_t^i$ denotes the sample mean of the observed bang-per-buck ratios up to time $t$ and $e_{i,t}$ the exploration term defined in Algorithm \ref{thm:stochastic_case_sublinear_regret}. At each round, the $K$ arms associated with the $K$ largest confidence bounds are played. For initialization purposes, we allow all $N$ arms to be played exactly once prior to the while-loop.
\begin{theorem}\label{thm:stochastic_case_sublinear_regret}
There exist constants $c_1$, $c_2$, and $c_3$, which are functions of $N, K, c_{\min}, \Delta_{\min}, \mu_{i}, \mu_c$ only, such that Algorithm \ref{alg:Comb_MAB_Budget_Constrained} (\texttt{UCB-MB}) achieves expected regret
\begin{align}
\mathcal{R}_{\mathcal{A}}(B) \leq c_1 + c_2 \log(B + c_3) = O(NK^4 \log B).
\end{align}
\end{theorem}
In Theorem \ref{thm:stochastic_case_sublinear_regret}, $\Delta_{\min}$ denotes the smallest possible difference of bang-per-buck ratios among non-optimal selections $a\neq a^\ast$, i.e. the second best choice of arms:
\begin{align}
\Delta_{\min} = \sum_{j\in a^\ast} \mu^j - \max_{a, a\neq a^\ast} \sum_{j\in a}\mu^j.
\end{align}
Similarly, the proof of Theorem \ref{thm:stochastic_case_sublinear_regret} also relies on the largest such difference $\Delta_{\max}$, which corresponds to the worst possible choice of arms:
\begin{align}
\Delta_{\max} = \sum_{j\in a^\ast} \mu^j - \min_{a, a\neq a^\ast} \sum_{j\in a}\mu^j.
\end{align}
Comparing the bound given in Theorem \ref{thm:stochastic_case_sublinear_regret} to the results in Table \ref{tab:regret_bounds_overview}, we recover the $O(N\log B)$ bound from \cite{Ding:2013aa} for the single-play case.

\begin{algorithm}[h]
\caption{\texttt{UCB-MB} for Stochastic MAB}
\label{alg:Comb_MAB_Budget_Constrained}
\textbf{Initialize:} $t=1$. Play all arms together exactly once. Let $\bar{\mu}_{r,1}^i = r_{i,1}$, $\bar{\mu}_{c,1}^i = c_{i,1}$, $\bar{\mu}_1^i = \frac{\bar{\mu}_{r,1}^i}{\bar{\mu}_{c,1}^i}~\forall~i\in [N]$, $n_{i,1} = 1$, $e_{i,1} = 0~\forall~i\in [N]$, $G_{\mathcal{A}} = 0$.

\begin{algorithmic}[1]
\WHILE {true}
\STATE $a_t \gets$ Indexes of $K$ arms with $K$ largest $U_{i,t}$. \label{alg:UCB}
\IF {$\sum_{j\in a_t} c_{j,t} > B$}
\RETURN Gain $G_{\mathcal{A}}$, stopping time $\tau_{\mathcal{A}}(B) = t$ \label{alg:return}
\ENDIF{}
\STATE $G_{\mathcal{A}} \gets G_{\mathcal{A}} + \sum_{i\in a_t}r_{i,t}$, $\quad B \gets B - \sum_{i\in a_t} c_{i,t}$
\STATE $n_{i,t} \gets n_{i,t} + 1\quad \forall~i\in a_t$
\STATE $t\gets t+1$
\STATE $e_{i,t} \gets \frac{\sqrt{(K+1)\log t / n_{i,t}}(1+1/c_{\min})}{c_{\min} - \sqrt{(K+1)\log t / n_{i,t}}}$\label{alg:UCB_update}
\ENDWHILE{}
\end{algorithmic}
\end{algorithm}

\subsection{Adversarial Setting}\label{sec:adversarial_main}
We now consider the adversarial case that makes no assumptions on the reward and cost distributions whatsoever. The setup for this case was first proposed and analyzed by \cite{Auer:2002aa} for the single play case (i.e. $K = 1$), a fixed horizon $T$, and an oblivious adversary. That is, the entire seqence of rewards for all arms is fixed in advance and in particular cannot be adaptively changed during runtime. The proposed randomized algorithm \texttt{Exp3} enjoys $O(\sqrt{NT\log N})$ regret. Under \textit{semi-bandit} feedback, where the rewards for a given round are observed for each arm played, \cite{Uchiya:2010aa} derived a variation of the single-play \texttt{Exp3} algorithm, which they called \texttt{Exp3.M} and enjoys regret $O\left(\sqrt{NTK\log(N/K)} \right)$, where $K$ is the number of plays per round.

We consider the extension of the classic setting as in \cite{Uchiya:2010aa}, where the decision maker has to play exactly $1\leq K \leq N$ arms. For each arm $i$ played at round $t$, the player observes the reward $r_i(t)\in [0, 1]$ and, unlike in previous settings, additionally the cost $0 < c_{\min} < c_i(t) < 1$. As in the stochastic setting (Section \ref{sec:stochastic_main}), the player is given a budget $B > 0$ to pay for the costs incurred, and the algorithm terminates after $\tau_{\mathcal{A}}(B)$ rounds when the sum of materialized costs in round $\tau_{\mathcal{A}}(B)$ exceeds the remaining budget. The gain $G_{\mathcal{A}}(B)$ of algorithm $\mathcal{A}$ is the sum of observed rewards up to and including round $\tau_{\mathcal{A}}(B) - 1$. The expected regret $\mathcal{R}_{\mathcal{A}}(B)$ is defined as in \eqref{eq:regret_definition}, where the gain of algorithm $\mathcal{A}$ is compared against the best set of arms that an omniscient algorithm $\mathcal{A}^\ast$, which knows the reward and cost sequences in advance, would select, given the budget $B$. In contrast to the stochastic case, the expectation is now taken with respect to algorithm $\mathcal{A}$'s internal randomness.
\subsubsection*{Upper Bounds on the Regret}
We begin with upper bounds on the regret for the budget constrained MAB with multiple plays and later transition towards lower bounds and upper bounds that hold with high probability. Algorithm \ref{alg:Adversarial_Budget_Constrained}, which we call \texttt{Exp3.M.B}, provides a randomized algorithm to achieve sublinear regret. Similar to the original \texttt{Exp3} algorithm developed by \cite{Auer:2002aa}, Algorithm \texttt{Exp3.M.B} maintains a set of time-varying weights $\lbrace w_i(t)\rbrace_{i=1}^N$ for all arms, from which the probabilities for each arm being played at time $t$ are calculated (line \ref{alg:calculate_probabilities_adversarial}). As noted in \cite{Uchiya:2010aa}, the probabilities $\lbrace p_i(t)\rbrace_{i=1}^N$ sum to $K$ (because exactly $K$ arms need to be played), which requires the weights to be capped at a value $v_t > 0$ (line \ref{alg:v_calculation_adversarial}) such that the probabilities $\lbrace p_i(t)\rbrace_{i=1}^N$ are kept in the range $[0, 1]$. In each round, the player draws a set of distinct arms $a_t$ of cardinality $|a_t|=K$, where each arm has probability $p_i(t)$ of being included in $a_t$ (line \ref{alg:dependent_rounding}). This is done by employing algorithm \texttt{DependentRounding} introduced by \cite{Gandhi:2006aa}, which runs in $O(K)$ time and $O(N)$ space. At the end of each round, the observed rewards and costs for the played arms are turned into estimates $\hat{r}_i(t)$ and $\hat{c}_i(t)$ such that $\mathbb{E}[\hat{r}_i(t)~|~a_t, \ldots, a_1] = r_i(t)$ and $\mathbb{E}[\hat{c}_i(t)~|~a_t, \ldots, a_1] = c_i(t)$ for $i\in a_t$ (line \ref{alg:update_weights_adversarial}). Arms with $w_i(t) < v_t$ are updated according to $(\hat{r}_i(t) - \hat{c}_i(t))$, which assigns larger weights as $\hat{r}_i(t)$ increases and $\hat{c}_i(t)$ decreases, as one might expect.

\begin{algorithm}[h!]
\caption{\texttt{Exp3.M.B}: Budget Constrained Multi-Armed Bandit, Multiple Play, Adversarial}
\label{alg:Adversarial_Budget_Constrained}
\textbf{Initialize:} $w_i = 1$ for $i\in [N]$, gain $G_{\mathcal{A}} = 0$.

\begin{algorithmic}[1]
\WHILE {$B > 0$}
\IF {$\arg\max_{i\in [N]} w_i(t) \geq \left(\frac{1}{K} - \frac{\gamma}{N} \right)\sum_{j=1}^N \frac{w_i(t)}{1-\gamma}$}
\STATE Determine $v_t$ as follows:~ $1/K - \gamma/N =$ 
\begin{align*}
\frac{v_t(1-\gamma)}{\sum_{i=1}^N v_t \cdot\mathds{1}(w_i(t) \geq v_t) + w_i(t)\cdot\mathds{1}(w_i(t) < v_t)}
\end{align*}\label{alg:v_calculation_adversarial}
\STATE Define set $\tilde{S}(t) = \lbrace i\in [N]~|~w_i(t) \geq v_t\rbrace$.
\STATE Define weights $\tilde{w}_i(t) = v_t$ for $i\in\tilde{S}(t)$.
\ELSE
\STATE Define set $\tilde{S}(t) = \lbrace\rbrace$.
\ENDIF{}
\STATE Define weights $\tilde{w}_i(t) = w_i(t)$ for $i\in [N]\setminus \tilde{S}(t)$.
\STATE Calculate probabilities for each $i\in [N]$:
\begin{align*}
p_i(t) = K\left( (1-\gamma)\frac{\tilde{w}_i(t)}{\sum_{j=1}^N \tilde{w}_j(t)} + \frac{\gamma}{N} \right).
\end{align*}\label{alg:calculate_probabilities_adversarial}
\STATE Play arms $a_t \sim p_1, \ldots, p_N$.\label{alg:dependent_rounding}
\IF {$\sum_{i\in a_t} c_i(t) > B$}
\RETURN Gain $G_{\texttt{Exp3.M.B}}$, stopping time $\tau_{\mathcal{A}}(B) = t$ \label{alg:return_adversarial}
\ENDIF{}
\STATE $B\leftarrow B - \sum_{i\in a_t} c_i(t)$, $~G_{\mathcal{A}} \leftarrow G_{\mathcal{A}} + \sum_{i\in a_t} r_i(t)$.
\STATE Calculate estimated rewards and costs to update weights for each $i\in [N]$: 
\begin{align*}
\hat{r}_i(t) &= r_i(t) / p_i(t) \cdot \mathds{1}(i\in a_t)\\
\hat{c}_i(t) &= c_i(t) / p_i(t) \cdot \mathds{1}(i\in a_t)\\
w_i(t+1) &= w_i(t)\exp\left[  \frac{K\gamma}{N}\left[\hat{r}_i(t) - \hat{c}_i(t)\right] \mathds{1}_{i\in \tilde{S}(t)} \right]
\end{align*}\label{alg:update_weights_adversarial}
\ENDWHILE{}
\end{algorithmic}
\end{algorithm}

\begin{theorem}\label{thm:adversarial_case_bound}
Algorithm \texttt{Exp3.M.B} achieves regret
\begin{align}\label{eq:adversarial_bound_i}
\mathcal{R} \leq 2.63\sqrt{1 + \frac{B}{g c_{\min}}} \sqrt{gN\log(N/K)} + K,
\end{align}
where $g$ is an upper bound on $G_{\max}$, the maximal gain of the optimal algorithm. This bound is of order $O(\sqrt{BN\log(N/K)})$.
\end{theorem}
The runtime of Algorithm \texttt{Exp3.M.B} and its space complexity is linear in the number of arms, i.e. $O(N)$. If no bound $g$ on $G_{\max}$ exists, we have to modify Algorithm \ref{alg:Adversarial_Budget_Constrained}. Specifically, the weights are now updated as follows:
\begin{align}
\hspace*{-0.1cm} w_i(t+1) &= w_i(t)\exp\left[  \frac{K\gamma}{N}\left[\hat{r}_i(t) - \hat{c}_i(t)\right]\cdot \mathds{1}_{i\in a_t} \right].
\end{align}
This replaces the original update step in line \ref{alg:update_weights_adversarial} of Algorithm \ref{alg:Adversarial_Budget_Constrained}. As in Algorithm \texttt{Exp3.1} in \cite{Auer:2002aa}, we use an adaptation of Algorithm \ref{alg:Adversarial_Budget_Constrained}, which we call \texttt{Exp3.1.M.B}, see Algorithm \ref{alg:Exp3.1.M.B}. In Algorithm \ref{alg:Exp3.1.M.B}, we define cumulative expected gains and losses
\begin{subequations}
\begin{align}
\hat{G}_i(t) &= \sum_{s=1}^t \hat{r}_i(s),\\
\hat{L}_i(t) &= \sum_{s=1}^t \hat{c}_i(s).
\end{align}
\end{subequations}
and make the following, necessary assumption:
\begin{assumption}\label{as:incentive_compatibility}
$\sum_{i\in a} r_i(t) \geq \sum_{i\in a} c_i(t)$ for all $a \in \mathcal{S}$ possible $K$-combinations and $t \geq 1$.
\end{assumption}
Assumption \ref{as:incentive_compatibility} is a natural assumption, which is motivated by ``individual rationality'' reasons. In other words, a user will only play the bandit algorithm if the reward at any given round, for any possible choice of arms, is at least as large as the cost that incurs for playing. Under the caveat of this assumption, Algorithm \texttt{Exp3.1.M.B} utilizes Algorithm \texttt{Exp3.1.M} as a subroutine in each epoch until termination.

\begin{algorithm}[h]
\caption{Algorithm \texttt{Exp3.1.M.B} with Budget $B$}
\label{alg:Exp3.1.M.B}
\textbf{Initialize:} $t=1$, $w_i = 1$ for $i\in [N]$, $r=0$.

\begin{algorithmic}[1]
\WHILE {$\sum_{t=1}^T \sum_{i\in a_t} c_i(t) \leq B$}
\STATE Define $g_r = \frac{N\log(N/K)}{(e-1)-(e-2)c_{\min}}4^r$ \label{alg:g_r_definition_exp31mB}
\STATE Restart \texttt{Exp3.M.B} with $\gamma_r = \min\left(1, 2^{-r} \right)$
\WHILE {$\max_{a\in\mathcal{S}} \sum_{i\in a}(\hat{G}_i(t) - \hat{L}_i(t)) \leq g_r - \frac{N(1-c_{\min})}{K\gamma_r}$}\label{alg:termination_criterion_exp31mB}
\STATE Draw $a_t\sim p_1, \ldots, p_N$, observe $r_i(t)$ and $c_i(t)$ for $i\in a_t$, calculate $\hat{r}_i(t)$ and $\hat{c}_i(t)$.
\STATE $\hat{G}_i(t+1) \gets \hat{G}_i(t) + \hat{r}_i(t)$ for $i \in [N]$
\STATE $\hat{L}_i(t+1) \gets \hat{L}_i(t) + \hat{c}_i(t)$ for $i \in [N]$
\STATE $t\gets t + 1$
\ENDWHILE {}
\ENDWHILE {}
\RETURN Gain $G_{\texttt{Exp3.1.M.B}}$
\end{algorithmic}
\end{algorithm}

\begin{proposition}\label{prop:upper_bound_no_g}
For the multiple plays case with budget, the regret of Algorithm \texttt{Exp3.1.M.B} is upper bounded by
\begin{align}
\mathcal{R} \leq 8\left[(e-1)-(e-2)c_{\min}\right]\frac{N}{K} + 2N\log\frac{N}{K}+ K +\nonumber\\
8 \sqrt{\left[(e-1)-(e-2)c_{\min}\right](G_{\max}-B+K)N\log(N/K)}\label{eq:adversarial_bound_ii}
\end{align}
\end{proposition}
This bound is of order $O((G_{\max} - B)N\log(N/K))$ and, due to Assumption \ref{as:incentive_compatibility}, not directly comparable to the bound in Theorem \ref{thm:adversarial_case_bound}. One case in which \eqref{eq:adversarial_bound_ii} outperforms \eqref{eq:adversarial_bound_i} occurs whenever only a loose upper bound of $g$ on $G_{\max}$ exists or whenever $G_{\max}$, the return of the best selection of arms, is ``small''.


\subsubsection*{Lower Bound on the Regret}
Theorem \ref{thm:lower_bound_multiple_play} provides a lower bound of order $\Omega((1-K/N)^2\sqrt{NB/K})$ on the weak regret of algorithm \texttt{Exp3.M.B}.
\begin{theorem}\label{thm:lower_bound_multiple_play}
For $1\leq K \leq N$, the weak regret $\mathcal{R}$ of Algorithm \texttt{Exp3.M.B} is lower bounded as follows:
\begin{align}\label{eq:lower_bound_regret_multiple_play_eps}
\mathcal{R} \geq \varepsilon\left( B - \frac{BK}{N} - 2Bc_{\min}^{-3/2}\varepsilon\sqrt{\frac{BK\log(4/3)}{N}} \right),
\end{align}
where $\varepsilon \in (0, 1/4]$. Choosing $\varepsilon$ as
\begin{align*}
\varepsilon = \min\left(\frac{1}{4},~\frac{(1-K/N)c_{\min}^{3/2}}{4\sqrt{\log(4/3)}}\sqrt{\frac{N}{BK}}\right)
\end{align*}
yields the bound
\begin{align}\label{eq:lower_bound_regret_multiple_play}
\mathcal{R} \geq \min\left(\frac{c_{\min}^{3/2}(1-K/N)^2}{8\sqrt{\log(4/3)}}\sqrt{\frac{NB}{K}} ,~\frac{B(1-K/N)}{8}\right).
\end{align}
\end{theorem}
This lower bound differs from the upper bound in Theorem \ref{thm:stochastic_case_sublinear_regret} by a factor of $\sqrt{K\log(N/K)}(N/(N-K))^2$. For the single-play case $K=1$, this factor is $\sqrt{\log N}$, which recovers the gap from \cite{Auer:2002aa}.

\subsubsection*{High Probability Upper Bounds on the Regret}
For a fixed number of rounds (no budget considerations) and single play per round ($K=1$), \cite{Auer:2002aa} proposed Algorithm \texttt{Exp3.P} to derive the following upper bound on the regret that holds with probability at least $1-\delta$:
\begin{align}
G_{\max} - &G_{\texttt{Exp3.P}} \leq 4\sqrt{NT \log\left( NT/\delta\right)} \nonumber \\
&+ 4\sqrt{\frac{5}{3}NT\log N} + 8\log\left( \frac{NT}{\delta}\right). \label{eq:high_probability_bound_single_play_fixed_T}
\end{align}
Theorem \ref{thm:multiple_play_fixed_round_high_probability} extends the non-budgeted case to the multiple play case.
\begin{theorem}\label{thm:multiple_play_fixed_round_high_probability}
For the multiple play algorithm ($1\leq K \leq N$) and a fixed number of rounds $T$, the following bound on the regret holds with probability at least $1-\delta$:
\begin{align}\label{eq:high_probability_bound_multiple_play_fixed_T}
\mathcal{R} &= G_{\max} - G_{\texttt{Exp3.P.M}} \nonumber\\
&\hspace*{0.4cm}\leq 2\sqrt{5}\sqrt{NKT\log(N/K)} + 8\frac{N-K}{N-1}\log\left( \frac{NT}{\delta} \right) \nonumber\\
&\hspace*{0.8cm}+ 2(1+K^2)\sqrt{NT\frac{N-K}{N-1}\log\left( \frac{NT}{\delta} \right)}.
\end{align}
\end{theorem}
For $K=1$, \eqref{eq:high_probability_bound_multiple_play_fixed_T} recovers \eqref{eq:high_probability_bound_single_play_fixed_T} save for the constants, which is due to a better $\varepsilon$-tuning in this paper compared to \cite{Auer:2002aa}. Agreeing with intuition, this upper bound becomes zero for the edge case $K\equiv N$.

Theorem \ref{thm:multiple_play_fixed_round_high_probability} can be derived by using a modified version of Algorithm \ref{alg:Adversarial_Budget_Constrained}, which we name \texttt{Exp3.P.M}. The necessary modifications to \texttt{Exp3.M.B} are motivated by Algorithm \texttt{Exp3.P} in \cite{Auer:2002aa} and are provided in the following:
\begin{itemize}
\item Replace the outer while loop with \textbf{for} $t = 1, \ldots, T$ \textbf{do}
\item Initialize parameter $\alpha$:
\begin{align*}
\alpha = 2\sqrt{(N-K)/(N-1) \log\left( NT/\delta \right)}.
\end{align*}
\item Initialize weights $w_i$ for $i\in [N]$:
\begin{align*}
w_i(1) = \exp\left( \alpha \gamma K^2 \sqrt{T/N} / 3 \right).
\end{align*}
\item Update weights for $i\in [N]$ as follows:
\begin{align}
&w_i(t+1) = w_i(t) \nonumber\\
&\hspace*{0.5cm} \times\exp\left[\mathds{1}_{i\not\in \tilde{S}(t)}\frac{\gamma K}{3N}\left(\hat{r}_i(t) + \frac{\alpha}{p_i(t)\sqrt{NT}}\right) \right].\label{eq:weight_update_exp3pm}
\end{align}
\end{itemize}
Since there is no notion of cost in Theorem \ref{thm:multiple_play_fixed_round_high_probability}, we do not need to update any cost terms.

Lastly, Theorem \ref{thm:multiple_play_budget_high_probability} extends Theorem \ref{thm:multiple_play_fixed_round_high_probability} to the budget constrained setting using algorithm \texttt{Exp3.P.M.B}.
\begin{theorem}\label{thm:multiple_play_budget_high_probability}
For the multiple play algorithm ($1\leq K \leq N$) and the budget $B > 0$, the following bound on the regret holds with probability at least $1-\delta$:
\begin{align}\label{eq:high_probability_bound_budget}
\mathcal{R} &= G_{\max} - G_{\texttt{Exp3.P.M.B}} \nonumber \\
&\leq 2\sqrt{3}\sqrt{\frac{NB(1-c_{\min})}{c_{\min}}\log\frac{N}{K}} \nonumber\\
&\hspace*{0.4cm}+ 4\sqrt{6}\frac{N-K}{N-1}\log\left( \frac{NB}{Kc_{\min}\delta} \right) \\
&\hspace*{0.4cm}+ 2\sqrt{6}(1+K^2)\sqrt{\frac{N-K}{N-1}\frac{NB}{Kc_{\min}}\log\left( \frac{NB}{Kc_{\min}\delta} \right)}.\nonumber
\end{align}
\end{theorem}
To derive bound \eqref{eq:high_probability_bound_budget}, we again modify the following update rules in Algorithm \ref{alg:Adversarial_Budget_Constrained} to obtain Algorithm \texttt{Exp3.P.M.B}:
\begin{itemize}
\item Initialize parameter $\alpha$:
\begin{align*}
\alpha = 2\sqrt{6}\sqrt{(N-K)/(N-1) \log\left( NB/(Kc_{\min}\delta)\right)}.
\end{align*}
\item Initialize weights $w_i$ for $i\in [N]$:
\begin{align*}
w_i(1) = \exp\left( \alpha \gamma K^2 \sqrt{B/(NKc_{\min})} / 3 \right).
\end{align*}
\item Update weights for $i\in [N]$ as follows:
\begin{align*}
&w_i(t+1) = w_i(t) \\
&\hspace*{0.5cm} \times\exp\left[ \mathds{1}_{i\not\in \tilde{S}(t)} \frac{\gamma K}{3N}\left(\hat{r}_i(t) - \hat{c}_i(t) + \frac{\alpha\sqrt{Kc_{\min}}}{p_i(t)\sqrt{NB}}\right) \right].
\end{align*}
\end{itemize}
The estimated costs $\hat{c}_i(t)$ are computed as $\hat{c}_i(t) = c_i(t)/p_i(t)$ whenever arm $i$ is played at time $t$, as is done in Algorithm \ref{alg:Adversarial_Budget_Constrained}.

%
%

\section{Proofs}\label{sec:proofs}

\subsection*{Proof of Theorem \ref{thm:stochastic_case_sublinear_regret}}
The proof of Theorem \ref{thm:stochastic_case_sublinear_regret} is divided into two technical lemmas introduced in the following. Due to space constraints, the proofs are relegated to the supplementary document.

First, we bound the number of times a non-optimal selection of arms is made up to stopping time $\tau_{\mathcal{A}}(B)$. For this purpose, let us define a counter $C_{i,t}$ for each arm $i$, initialized to zero for $t=1$. Each time a non-optimal vector of arms is played, that is, $a_t \neq a^\ast$, we increment the smallest counter in the set $a_t$:
\begin{align}\label{eq:increment_smallest_counter}
C_{j,t}\leftarrow C_{j,t}+1, \quad j = \arg\min_{i\in a_t}C_{i,t}.
\end{align}
Ties are broken randomly. By definition, the number of times arm $i$ has been played until time $t$ is greater than or equal to its counter $C_{i,t}$. Further, the sum of all counters is exactly the number of suboptimal choices made so far:
\begin{align*}
n_{i,t} &\geq C_{i,t}\quad\forall~i\in[N],~t\in[\tau_{\mathcal{A}}(B)].\\
\sum_{i=1}^N C_{i,t} &= \sum_{\tau=1}^t \mathds{1}(a_{\tau} \neq a^\ast)\quad\forall~t\in[\tau_{\mathcal{A}}(B)].
\end{align*}
Lemma \ref{lem:number_suboptimal_actions_bound} bounds the value of $C_{i,t}$ from above.
\begin{lemma}\label{lem:number_suboptimal_actions_bound}
Upon termination of algorithm $\mathcal{A}$, there have been at most $O\left( NK^3\log\tau_{\mathcal{A}}(B) \right)$ suboptimal actions. Specifically, for each $i\in[N]$:
\begin{align*}
\mathbb{E}&\left[C_{i,\tau_{\mathcal{A}}(B)}\right] \leq 1 + K\frac{\pi^2}{3} \\
& +(K+1)\left( \frac{\Delta_{\min} + 2K(1+1/c_{\min})}{c_{\min} \Delta_{\min}} \right)^2 \log \tau_{\mathcal{A}}(B).\nonumber
\end{align*}
\end{lemma}

Secondly, we relate the stopping time of algorithm $\mathcal{A}$ to the optimal action $a^\ast$:
\begin{lemma}\label{lem:stopping_time_explicit}
The stopping time $\tau_{\mathcal{A}}$ is bounded as follows:
\begin{align*}
\frac{B}{\sum_{i\in a^\ast} \mu_c^i} &- c_2 - c_3 \log\left( c_1 + \frac{2B}{\sum_{i\in a^\ast}\mu_c^i} \right) \\
&\leq \tau_{\mathcal{A}} \leq \frac{2B}{\sum_{i\in a^\ast}}\mu_c^i + c_1,
\end{align*}
where $c_1$, $c_2$, and $c_3$ are the same positive constants as in Theorem \ref{thm:stochastic_case_sublinear_regret} that depend only on $N, K, c_{\min}, \Delta_{\min}, \mu_c^i, \mu_r^i$.
\end{lemma}

Utilizing Lemmas \ref{lem:number_suboptimal_actions_bound} and \ref{lem:stopping_time_explicit} in conjunction with the definition of the weak regret \eqref{eq:regret_definition} yields Theorem \ref{thm:stochastic_case_sublinear_regret}. See the supplementary document for further technicalities.

\subsection*{Proof of Theorem \ref{thm:adversarial_case_bound}}
The proof of Theorem \ref{thm:adversarial_case_bound} in influenced by the proof methods for Algorithms \texttt{Exp3} by \cite{Auer:2002aa} and \texttt{Exp3.M} by \cite{Uchiya:2010aa}. The main challenge is the absence of a well-defined time horizon $T$ due to the time-varying costs. To remedy this problem, we define $T = \max\left( \tau_{\mathcal{A}}(B), \tau_{\mathcal{A}^\ast}(B) \right)$, which allows us to first express the regret as a function of $T$. In a second step, we relate $T$ to the budget $B$.

\subsection*{Proof of Proposition \ref{prop:upper_bound_no_g}}
The proof of Proposition \ref{prop:upper_bound_no_g} is divided into the following two lemmas:
\begin{lemma}\label{lem:regret_suffered_selected_epoch_with_budget}
For any subset $a\in\mathcal{S}$ of $K$ unique elements from $[N]$, $1\leq K \leq N$:
\begin{align}
&\sum_{t=S_r}^{T_r} \sum_{i\in a_t} (r_i(t)-c_i(t)) \geq \sum_{i\in a} \sum_{t=S_r}^{T_r} (\hat{r}_j(t) - \hat{c}_j(t)) \label{eq:regret_suffered_selected_epoch_with_budget}\\
&\hspace*{0.5cm} - 2\sqrt{(e-1)-(e-2)c_{\min}}\sqrt{g_r N\log(N/K)}, \nonumber
\end{align}
where $S_r$ and $T_r$ denote the first and last time step at epoch $r$, respectively. 
\end{lemma}

\begin{lemma}\label{lem:num_epochs_exp31m_with_budget}
The total number of epochs $R$ is bounded by
\begin{align}\label{eq:num_epochs_exp31m_with_budget_upperbound}
2^{R-1} \leq \frac{N(1-c_{\min})}{Kc} + \sqrt{\frac{\hat{G}_{\max} - \hat{L}_{\max}}{c}} + \frac{1}{2},
\end{align}
where $c = \frac{N\log(N/K)}{(e-1) - (e-2)c_{\min}}$.
\end{lemma}

To derive Proposition \ref{prop:upper_bound_no_g}, we combine Lemmas \ref{lem:regret_suffered_selected_epoch_with_budget} and \ref{lem:num_epochs_exp31m_with_budget} and utilize the fact that algorithm \texttt{Exp3.1.M.B} terminates after $\tau_{\mathcal{A}}(B)$ rounds. See supplementary document for details.

\subsection*{Proof of Theorem \ref{thm:lower_bound_multiple_play}}
The proof follows existing procedures for deriving lower bounds in adversarial bandit settings, see \cite{Auer:2002aa}, \cite{Cesa-Bianchi:2006aa}. The main challenges are found in generalizing the single play setting to the multiple play setting ($K>1$) as well as incorporating a notion of cost associated with bandits.

Select exactly $K$ out of $N$ arms at random to be the arms in the ``good'' subset $a^\ast$. For these arms, let $r_i(t)$ at each round $t$ be Bernoulli distributed with bias $\frac{1}{2}+\varepsilon$, and the cost $c_i(t)$ attain $c_{\min}$ and $1$ with probability $\frac{1}{2}+\varepsilon$ and $\frac{1}{2}-\varepsilon$, respectively, for some $0 < \varepsilon < 1/2$ to be specified later. All other $N-K$ arms are assigned rewards $0$ and $1$ and costs $c_{\min}$ and $1$ independently at random. Let $\mathbb{E}_{a^\ast}[\hspace*{0.05cm}\cdot\hspace*{0.05cm}]$ denote the expectation of a random variable conditional on $a^\ast$ as the set of good arms. Let $\mathbb{E}_u[\hspace*{0.05cm}\cdot\hspace*{0.05cm}]$ denote the expectation with respect to a uniform assignment of costs $\lbrace c_{\min}, 1\rbrace$ and rewards $\lbrace 0, 1\rbrace$ to all arms. Lemma \ref{lem:function_on_reward_cost_sequences} is an extension of Lemma A.1 in \cite{Auer:2002aa} to the multiple-play case with cost considerations:

\begin{lemma}\label{lem:function_on_reward_cost_sequences}
Let $f : \lbrace \lbrace 0, 1\rbrace, \lbrace c_{\min}, 1\rbrace\rbrace^{\tau_{\max}}\to [0, M]$ be any function defined on reward and cost sequences $\lbrace \mathbf{r}, \mathbf{c}\rbrace$ of length less than or equal $\tau_{\max} = \frac{B}{Kc_{\min}}$. Then, for the best action set $a^\ast$:
\begin{align*}
&\mathbb{E}_{a^\ast}\left[ f(\mathbf{r}, \mathbf{c}) \right] \\
&\hspace*{0.5cm}\leq \mathbb{E}_{u}[f(\mathbf{r}, \mathbf{c})] + \frac{Bc_{\min}^{-3/2}}{2}\sqrt{-\mathbb{E}_{u}[N_{a^\ast}]\log(1-4\varepsilon^2)}, \nonumber
\end{align*}
\end{lemma}
where $N_{a^\ast}$ denotes the total number of plays of arms in $a^\ast$ during rounds $t=1$ through $t=\tau_{\mathcal{A}}(B)$, that is:
\begin{align*}
N_{a^\ast} = \sum_{t=1}^{\tau_{\mathcal{A}}(B)}\sum_{i\in a^\ast} \mathds{1}\left( i\in a_t \right).
\end{align*}
Lemma \ref{lem:function_on_reward_cost_sequences}, whose proof is relegated to the supplementary document, allows us to bound the gain under the existence of $K$ optimal arms by treating the problem as a uniform assignment of costs and rewards to arms. The technical parts of the proof can also be found in the supplementary document.

\subsection*{Proof of Theorem \ref{thm:multiple_play_fixed_round_high_probability}}
The proof strategy is to acknowledge that Algorithm \texttt{Exp3.P.M} uses upper confidence bounds $\hat{r}_i(t) + \frac{\alpha}{p_i(t)\sqrt{NT}}$ to update weights \eqref{eq:weight_update_exp3pm}. Lemma \ref{lem:confidence_level_exp3p} asserts that these confidence bounds are valid, namely that they upper bound the actual gain with probability at least $1-\delta$, where $0 < \delta \ll 1$.

\begin{lemma}\label{lem:confidence_level_exp3p}
For $2\sqrt{\frac{N-K}{N-1} \log\left( \frac{NT}{\delta} \right)} \leq \alpha \leq 2\sqrt{NT}$,
\begin{align*}
&\mathbb{P}\left( \hat{U}^\ast > G_{\max} \right) \nonumber\\
&\hspace*{0.4cm}\geq\mathbb{P}\left( \bigcap_{a\subset\mathcal{S}} \sum_{i\in a} \hat{G}_i + \alpha \hat{\sigma}_i > \sum_{i\in a} G_i \right) \geq 1-\delta,
\end{align*}
where $a\subset\mathcal{S}$ denotes an arbitrary subset of $K$ unique elements from $[N]$. $\hat{U}^\ast$ denotes the upper confidence bound for the optimal gain.
\end{lemma}

Next, Lemma \ref{lem:gain_exp3pm_bound} provides a lower bound on the gain of Algorithm \texttt{Exp3.P.M} as a function of the maximal upper confidence bound.
\begin{lemma}\label{lem:gain_exp3pm_bound}
For $\alpha \leq 2\sqrt{NT}$, the gain of Algorithm \texttt{Exp3.P.M} is bounded below as follows:
\begin{align}\label{eq:gain_exp3pm_bound}
G_{\texttt{Exp3.P.M}} &\geq \left( 1 - \frac{5}{3}\gamma\right) \hat{U}^\ast - \frac{3N}{\gamma} \log(N/K) \nonumber\\
&\hspace*{0.5cm}- 2\alpha^2 - \alpha(1 + K^2) \sqrt{NT},
\end{align}
where $\hat{U}^\ast = \sum_{j\in a^\ast}\hat{G}_j + \alpha \hat{\sigma}_j$ denotes the upper confidence bound of the optimal gain achieved with optimal set $a^\ast$.
\end{lemma}

Therefore, combining Lemmas \ref{lem:confidence_level_exp3p} and \ref{lem:gain_exp3pm_bound} upper bounds the actual gain of Algorithm \texttt{Exp3.P.M} with high probability. See the supplementary document for technical details.

\subsection*{Proof of Theorem \ref{thm:multiple_play_budget_high_probability}}
The proof of Theorem \ref{thm:multiple_play_budget_high_probability} proceeds in the same fashion as in Theorem \ref{thm:multiple_play_fixed_round_high_probability}. Importantly, the upper confidence bounds now include a cost term. Lemma \ref{lem:confidence_level_exp3p_budget} is the equivalent to Lemma \ref{lem:confidence_level_exp3p} for the budget constrained case:
\begin{lemma}\label{lem:confidence_level_exp3p_budget}
For $2\sqrt{6}\sqrt{\frac{N-K}{N-1}\log\frac{NB}{Kc_{\min}\delta}} \leq \alpha \leq 12\sqrt{\frac{NB}{Kc_{\min}}}$,
\begin{align*}
&\mathbb{P}\left( \hat{U}^\ast > G_{\max} - B \right) \nonumber \\
&\hspace*{0.26cm}\geq\mathbb{P}\left( \bigcap_{a\subset\mathcal{S}} \sum_{i\in a} \hat{G}_i - \hat{L}_i + \alpha \hat{\sigma}_i > \sum_{i\in a} G_i - L_i \right) \geq 1-\delta,
\end{align*}
where $a\subset\mathcal{S}$ denotes an arbitrary time-invariant subset of $K$ unique elements from $[N]$. $\hat{U}^\ast$ denotes the upper confidence bound for the cumulative optimal gain minus the cumulative cost incurred after $\tau_a(B)$ rounds (the stopping time when the budget is exhausted):
\begin{align}
a^\ast &= \max_{a\in\mathcal{S}}\sum_{t=1}^{\tau_a(B)}(r_i(t) - c_i(t)), \nonumber\\
\hat{U}^\ast &= \sum_{i\in a^\ast}\left(\alpha\hat{\sigma}_i +\sum_{t=1}^{\tau_{a^\ast}(B)}(\hat{r}_i(t) - \hat{c}_i(t))\right).\label{eq:Uhat_explanation}
\end{align}
\end{lemma}
In Lemma \ref{lem:confidence_level_exp3p_budget}, $G_{\max}$ denotes the optimal cumulative reward under the optimal set $a^\ast$ chosen in \eqref{eq:Uhat_explanation}. $\hat{G}_i$ and $\hat{L}_i$ denote the cumulative expected reward and cost of arm $i$ after exhaustion of the budget (that is, after $\tau_{a}(B)$ rounds), respectively.

Lastly, Lemma \ref{lem:gain_exp3pm_bound_budget} lower bounds the actual gain of Algorithm \texttt{Exp3.P.M.B} as a function of the upper confidence bound \eqref{eq:Uhat_explanation}.

\begin{lemma}\label{lem:gain_exp3pm_bound_budget}
For $\alpha \leq 2\sqrt{\frac{NB}{Kc_{\min}}}$, the gain of Algorithm \texttt{Exp3.P.M.B} is bounded below as follows:
\begin{align*}
G_{\texttt{Exp3.P.M.B}} &\geq \left( 1 - \gamma - \frac{2\gamma}{3}\frac{1-c_{\min}}{c_{\min}}\right)\hat{U}^\ast \\
&\hspace*{0.4cm}- \frac{3N}{\gamma}\log\frac{N}{K} - 2\alpha^2  - \alpha(1 + K^2)\frac{BN}{Kc_{\min}}. \nonumber
\end{align*}
\end{lemma}
Combining Lemmas \ref{lem:confidence_level_exp3p_budget} and \ref{lem:gain_exp3pm_bound_budget} completes the proof, see the supplementary document.

%

%
%

\section{Discussion and Conclusion}
We discussed the budget-constrained multi-armed bandit problem with $N$ arms, $K$ multiple plays, and an a-priori defined budget $B$. We explored the stochastic as well as the adversarial case and provided algorithms to derive regret bounds in the budget $B$. For the stochastic setting, our algorithm \texttt{UCB-MB} enjoys regret $O(NK^4 \log B)$. In the adversarial case, we showed that algorithm \texttt{Exp3.M.B} enjoys an upper bound on the regret of order $O(\sqrt{NB\log(N/K)})$ and a lower bound $\Omega((1-K/N)^2\sqrt{NB/K})$. Lastly, we derived upper bounds that hold with high probability.

Our work can be extended in several dimensions in future research. For example, the incorporation of a budget constraint in this paper leads us to believe that a logical extension is to integrate ideas from economics, in particular mechanism design, into the multiple plays setting (one might think about auctioning off multiple items simultaneously) \cite{Babaioff:2009aa}. A possible idea is to investigate to which extent the regret varies as the number of plays $K$ increases. Further, we believe that in such settings, repeated interactions with customers (playing arms) give rise to strategic considerations, in which customers can misreport their preferences in the first few rounds to maximize their long-run surplus. While the works of \cite{Amin:2013aa} and \cite{Mohri:2014aa} investigate repeated interactions with a single player only, we believe an extension to a pool of buyers is worth exploring. In this setting, we would expect that the extent of strategic behavior decreases as the number of plays $K$ in each round increases, since the decision-maker could simply ignore users in future rounds who previously declined offers.

{\small{\bibliography{References}}}
\bibliographystyle{aaai}

\newpage
\appendix
\allowdisplaybreaks
%
%

\section{Proofs for Stochastic Setting}\label{sec:proofs_stochastic_long}
For convenience, we restate all theorems and lemmas before proving them.

\subsection*{Proof of Lemma \ref{lem:number_suboptimal_actions_bound}}
Recall the definition of counters $C_{i,t}$. Each time a non-optimal vector of arms is played, that is, $a_t \neq a^\ast$, we increment the smallest counter in the set $a_t$:
\begin{align}
C_{j,t}\leftarrow C_{j,t}+1, \quad j = \arg\min_{i\in a_t}C_{i,t}
\end{align}

\begin{duplicateLemma}
Upon termination of algorithm $\mathcal{A}$, there have been at most $O\left( NK^3\log\tau_{\mathcal{A}}(B) \right)$ suboptimal actions. Specifically, for each $i\in[N]$:
\begin{align*}
\mathbb{E}&\left[C_{i,\tau_{\mathcal{A}}(B)}\right] \leq 1 + K\frac{\pi^2}{3} \\
& +(K+1)\left( \frac{\Delta_{\min} + 2K(1+1/c_{\min})}{c_{\min} \Delta_{\min}} \right)^2 \log \tau_{\mathcal{A}}(B).\nonumber
\end{align*}
\end{duplicateLemma}


\begin{proof}
Let $I_i(t)$ denote the indicator that $C_{i,t}$ is incremented at time $t$. Then for any time $\tau$, we have:
\begin{align*}
C_{i,\tau} &= \sum_{t=2}^{\tau} \mathds{1}\left[ I_i(t) = 1 \right] = m + \sum_{t=2}^{\tau} \mathds{1}\left[ I_i(t) = 1,~C_{i,t}\geq m \right] \\
 &= m + \sum_{t=2}^{\tau} \mathds{1}\left[ \sum_{j\in a_t} \bar{\mu}_{t-1}^j + e_{j,t-1} \geq \right.\\
 &\hspace*{2.2cm}\left.\sum_{j\in a^\ast} \bar{\mu}_{t-1}^j + e_{j,t-1}~,~C_{i,t}\geq m\right] \\
 &\leq m + \sum_{t=1}^{\tau} \mathds{1}\left[ \sum_{j\in a_{t+1}} U_{j,t} \geq \sum_{j\in a^\ast} U_{j,t}~,~C_{i,t}\geq m\right] \\
 & \leq m + \sum_{t=1}^{\tau} \mathds{1}\left[ \max_{m\leq n_{s(1)}, \ldots, n_{s(K)} \leq t} \sum_{j=1}^{K} U_{n_{s(j),t}} \geq \right.\\
 &\hspace*{2.2cm}\left. \min_{1 \leq n_{s^\ast(1)}, \ldots, n_{s^\ast(K)} \leq t} \sum_{j=1}^K U_{n_{s^\ast(j),t}} \right]\\
 & \leq m + \sum_{t=1}^\infty \sum_{n_{s(1)}=m}^{t}\cdots \sum_{n_{s(K)}=m}^{t} \sum_{n_{s^\ast(1)}=1}^{t} \cdots \sum_{n_{s^\ast(K)}=1}^{t}\\
 &\hspace*{1.2cm} \mathds{1} \left[ \sum_{j=1}^K U_{n_{s(j),t}} \geq \sum_{j=1}^K U_{n_{s^\ast(j),t}} \right].
\end{align*}
$s(j)$ and $s^\ast(j)$ denote the $j$-th nonzero element in $a_{t+1}$ and $a^\ast$, respectively. $U_{n_{s(j),t}} = \bar{\mu}_t^{s(j)} + e_{s(j),t}$ is the upper confidence bound of arm $s(j)$ at time $t$ after it has been played $n_{s(j),t}$ times.

Using the choice of $m$ in Lemma A\ref{lem:probability_bad_choice}, we obtain the lower bound on the expectation of $C_{i,\tau(B)}$ as stated in Lemma \ref{lem:number_suboptimal_actions_bound}:
\begin{align}
\mathbb{E}&[C_{i,\tau(B)}] \nonumber\\
&\leq m + \sum_{t=1}^\infty \sum_{n_{s(1)}=m}^{t}\cdots \sum_{n_{s(K)}=m}^{t} \sum_{n_{s^\ast(1)}=1}^{t} \cdots \sum_{n_{s^\ast(K)}=1}^{t} \nonumber \\
&\hspace*{1.2cm} 2Kt^{-2(K+1)} \nonumber \\
& \leq \ceil*{(K+1)\log \tau(B) \left(\frac{\Delta_{\min} + 2K(1+1/c_{\min})}{c_{\min} \Delta_{\min}}\right)^2} \nonumber \\
&\hspace*{0.5cm} + \sum_{t=1}^\infty(t-m+1)^K t^K 2Kt^{-2(K+1)} \nonumber\\
&\leq (K+1) \left(\frac{\Delta_{\min} + 2K(1+1/c_{\min})}{c_{\min} \Delta_{\min}}\right)^2 \log\tau(B) \nonumber \\
&\hspace*{0.5cm}+ 1 + 2K\sum_{t=1}^\infty t^{-2} \nonumber\\
&\leq \underbrace{(K+1) \left(\frac{\Delta_{\min} + 2K(1+1/c_{\min})}{c_{\min} \Delta_{\min}}\right)^2}_{=:\gamma} \log\tau(B) \nonumber \\
&\hspace*{0.5cm}+ \underbrace{1 + K\frac{\pi^2}{3}}_{=: \delta}. \label{eq:lower_bound_Cit}
\end{align}
\end{proof}

\begin{lemmaA}\label{lem:probability_bad_choice}
For the choice
\begin{align*}
m \geq (K+1)\log \tau(B) \left[\frac{\Delta_{\min} + 2K(1+1/c_{\min})}{c_{\min} \Delta_{\min}}\right]^2
\end{align*}
we obtain the following bound:
\begin{align*}
\mathbb{P}&\left( \sum_{j=1}^K \bar{\mu}_t^{s(j)} + e_{s(j),t} \geq \sum_{j=1}^K \bar{\mu}_t^{s^\ast(j)} + e_{s^\ast(j),t} \right) \\
&\hspace*{0.5cm}\leq 2Kt^{-2(K+1)}.
\end{align*}
\end{lemmaA}

\begin{proof}
The proof follows ideas employed in \cite{Auer:2002ab}. Assuming that the event
\begin{align}\label{eq:event_true_assumption}
\sum_{j=1}^K \bar{\mu}_t^{s(j)} + e_{s(j),t} \geq \sum_{j=1}^K \bar{\mu}_t^{s^\ast(j)} + e_{s^\ast(j),t}
\end{align}
is true, at least one of the following events must also be true:
\begin{subequations}
\begin{align}
\sum_{j=1}^K \bar{\mu}_t^{s^\ast(j)} &\leq \sum_{j=1}^K \mu_{s^\ast(j)} - e_{s^\ast(j),t} \label{eq:first_event} \\
\sum_{j=1}^K \bar{\mu}_t^{s(j)} &\geq \sum_{j=1}^K \mu_{s(j)} + e_{s(j),t} \label{eq:second_event} \\
\sum_{j=1}^K \mu_{s^\ast(j)} &< \sum_{j=1}^K \mu_{s(j)} + 2e_{s(j),t} \label{eq:third_event}
\end{align}
\end{subequations}
To show this claim, assume the probabilities of events \eqref{eq:first_event} or \eqref{eq:second_event} occurring is zero. Then it follows that
\begin{align*}
\sum_{j=1}^K \bar{\mu}_{t}^{s(j)} &+ e_{s(j),t} \geq \sum_{j=1}^K \bar{\mu}_{t}^{s^\ast (j)} + e_{s^\ast(j),t} \\
& \stackrel{\eqref{eq:first_event}}{>} \mu_{s^\ast(j)} - e_{s^\ast(j),t} + e_{s^\ast(j),t} = \sum_{j=1}^K \mu_{s^\ast(j)}
\end{align*}
and
\begin{align*}
\sum_{j=1}^K \mu_{s(j)} + 2e_{s(j),t} &\stackrel{\eqref{eq:second_event}}{>} \sum_{j=1}^K \bar{\mu}_t^{s(j)} + e_{s(j),t} \\
&\geq \sum_{j=1}^K \bar{\mu}_t^{s^\ast(j)} + e_{s^\ast(j),t}.
\end{align*}
Hence, it follows that
\begin{align*}
\sum_{j=1}^K \mu_{s^\ast(j)} < \sum_{j=1}^K \bar{\mu}_{t}^{s^\ast (j)} + e_{s^\ast(j),t} < \sum_{j=1}^K \mu_{s(j)} + 2e_{s(j),t},
\end{align*}
which is exactly event \eqref{eq:third_event}. Thus, it suffices to upper-bound the probability of events \eqref{eq:first_event} and \eqref{eq:second_event}, while choosing $m$ such that the third event \eqref{eq:third_event} occurs with probability zero. Using Lemma A\ref{lem:prob_events_i_ii_bound}, we have
\begin{align}
\mathbb{P}\left( \ref{eq:first_event}~ \text{true}\right) + \mathbb{P}\left(\ref{eq:second_event}~ \text{true} \right) \leq 2Kt^{-2(K+1)}.
\end{align}
Now we pick $m$ such that event \eqref{eq:third_event} becomes impossible:
\begin{align*}
&\sum_{j=1}^K \mu_{s^\ast(j)} - \sum_{j=1}^K \mu_{s(j)} - \sum_{j=1}^K 2e_{s(j),t} \\
&\hspace*{0.4cm}=\sum_{j=1}^K \left[\mu_{s^\ast(j)} - \mu_{s(j)}\right] - 2\sum_{j=1}^K\frac{(1+1/c_{\min})\varepsilon_{s(j),t}}{c_{\min} - \varepsilon_{s(j),t}} \\
&=:~\Delta_{a_{t+1}} - 2\sum_{j=1}^K\frac{(1+1/c_{\min})\varepsilon_{s(j),t}}{c_{\min} - \varepsilon_{s(j),t}} \\
&= \Delta_{a_{t+1}} - 2\sum_{j=1}^K\frac{(1+1/c_{\min})\sqrt{\frac{(K+1)\log t}{n_{s(j)}}}}{c_{\min} - \sqrt{\frac{(K+1)\log t}{n_{s(j)}}}} \\
&\geq \Delta_{a_{t+1}} - 2K\frac{(1+1/c_{\min})\sqrt{\frac{(K+1)\log \tau(B)}{m}}}{c_{\min} - \sqrt{\frac{(K+1)\log \tau(B)}{m}}} \geq 0,
\end{align*}
where the last inequality is obtained by selecting $m$ as follows:
\begin{align*}
m \geq (K+1)\log \tau(B) \left(\frac{\Delta_{a_{t+1}} + 2K(1+1/c_{\min})}{c_{\min} \Delta_{a_{t+1}}}\right)^2.
\end{align*}
This choice of $m$ is suitable for the \textit{particular} choice of $a_{t+1}$. To falsify \eqref{eq:third_event} for all possible choices of $a_{t+1}$, we let $m$ be defined as follows:
\begin{align*}
m \geq (K+1)\log \tau(B) \left(\frac{\Delta_{\min} + 2K(1+1/c_{\min})}{c_{\min} \Delta_{\min}}\right)^2.
\end{align*}
\end{proof}

\begin{lemmaA}\label{lem:prob_events_i_ii_bound}
The probabilities of the events \eqref{eq:first_event} and \eqref{eq:second_event} are upper-bounded as follows:
\begin{subequations}
\begin{align}
\mathbb{P}\left( \sum_{j=1}^K \bar{\mu}_t^{s^\ast(j)} \leq \sum_{j=1}^K \mu_{s^\ast(j)} - e_{s^\ast(j),t} \right) &\leq Kt^{-2(K+1)} \label{eq:lemma_prob_bound_statement_i}\\
\mathbb{P}\left( \sum_{j=1}^K \bar{\mu}_t^{s(j)} \geq \sum_{j=1}^K \mu_{s(j)} + e_{s(j),t} \right) &\leq Kt^{-2(K+1)}\label{eq:lemma_prob_bound_statement_ii}
\end{align}
\end{subequations}
\end{lemmaA}

\begin{proof}
Using the union bound on \eqref{eq:lemma_prob_bound_statement_i}, we obtain
\begin{align*}
&\mathbb{P}\left( \sum_{j=1}^K \bar{\mu}_t^{s^\ast(j)} \leq \sum_{j=1}^K \mu_{s^\ast(j)} - e_{s^\ast(j),t} \right) \\
&\hspace*{0.4cm}\leq \sum_{j=1}^K \mathbb{P} \left( \bar{\mu}_t^{s^\ast(j)} \leq \mu_{s^\ast(j)} - e_{s^\ast(j),t} \right)
\end{align*}
Analyzing
\begin{align}
\bar{\mu}_t^{s^\ast(j)} &\leq \mu_{s^\ast(j)} - e_{s^\ast(j),t} \nonumber \\
\Leftrightarrow\frac{\bar{\mu}_{r,t}^{s^\ast(j)}}{\bar{\mu}_{c,t}^{s^\ast(j)}} &\leq \frac{\mu_r^{s^\ast(j)}}{\mu_c^{s^\ast(j)}} - e_{s^\ast(j),t} \label{eq:stoch_ratios_ineq}
\end{align}
we claim that at least one of the following two events must be true:
\begin{subequations}
\begin{align}
\bar{\mu}_{r,t}^{s^\ast(j)} &\leq \mu_r^{s^\ast(j)} - \varepsilon_{s^\ast(j),t} \label{eq:lowlevel_event1} \\
\bar{\mu}_{c,t}^{s^\ast(j)} &\geq \mu_c^{s^\ast(j)} + \varepsilon_{s^\ast(j),t} \label{eq:lowlevel_event2}
\end{align}
\end{subequations}
where $\varepsilon_{s^\ast(j), t}$ is the low-level exploration term for the mean reward and cost defined in \eqref{eq:exploration_term_low_level}. The claim is true, because if both \eqref{eq:lowlevel_event1} and \eqref{eq:lowlevel_event2} were false, then we would have
\begin{align*}
&\frac{\mu_r^{s^\ast(j)}}{\mu_c^{s^\ast(j)}} - \frac{\bar{\mu}_{r,t}^{s^\ast(j)}}{\bar{\mu}_{c,t}^{s^\ast(j)}} \\
&= \frac{\left( \mu_r^{s^\ast(j)}-\bar{\mu}_{r,t}^{s^\ast(j)} \right)\mu_c^{s^\ast(j)} - \left( \mu_c^{s^\ast(j)} - \bar{\mu}_{c,t}^{s^\ast(j)} \right)\mu_r^{s^\ast(j)}}{\bar{\mu}_{c,t}^{s^\ast(j)} \mu_c^{s^\ast(j)}}\\
&< \frac{\varepsilon_{s^\ast(j),t}}{\bar{\mu}_{c,t}^{s^\ast(j)}} + \frac{\varepsilon_{s^\ast(j),t}\mu_r^{s^\ast(j)}}{\bar{\mu}_{c,t}^{s^\ast(j)} \mu_c^{s^\ast(j)}}\\
&\leq \frac{\varepsilon_{s^\ast(j),t}}{c_{\min}} + \frac{\varepsilon_{s^\ast(j),t}\cdot 1}{c_{\min}^2} = \frac{\varepsilon_{s^\ast(j),t}(c_{\min} + 1)}{c_{\min}^2}\\
&\leq \frac{\varepsilon_{s^\ast(j),t}(1 + 1/c_{\min})}{c_{\min} - \varepsilon_{s^\ast(j),t}} \stackrel{!}{=} e_{s^\ast(j),t},
\end{align*}
which contradicts the claim \eqref{eq:stoch_ratios_ineq}. Now, choosing $\varepsilon_{s^\ast(j),t}$ as
\begin{align}\label{eq:exploration_term_low_level}
\varepsilon_{s^\ast(j),t} = \sqrt{\frac{(K+1)\log t}{n_{s^\ast(j)}}}
\end{align}
allows us to bound the probability of \eqref{eq:lowlevel_event1} and \eqref{eq:lowlevel_event2} using Hoeffding's Inequality:
\begin{align*}
\mathbb{P}\left( \bar{\mu}_{r,t}^{s^\ast(j)} \leq \mu_r^{s^\ast(j)} - \varepsilon_{s^\ast(j),t} \right) &\leq \exp\left(-2 n_{s^\ast(j)} \varepsilon_{s^\ast(j),t}^2\right) \\ &= t^{-2(K+1)}, \\
\mathbb{P}\left( \bar{\mu}_{c,t}^{s^\ast(j)} \geq \mu_c^{s^\ast(j)} + \varepsilon_{s^\ast(j),t} \right) &\leq \exp\left(-2 n_{s^\ast(j)} \varepsilon_{s^\ast(j),t}^2\right) \\ &= t^{-2(K+1)}.
\end{align*}
From here \eqref{eq:lemma_prob_bound_statement_i} follows. In a similar fashion, we can bound the probability of event \eqref{eq:lemma_prob_bound_statement_ii} by showing that at least one of
\begin{subequations}
\begin{align}
\bar{\mu}_{r,t}^{s(j)} &\geq \mu_r^{s(j)} + \varepsilon_{s(j),t} \label{eq:lowlevel_event1_two} \\
\bar{\mu}_{c,t}^{s(j)} &\leq \mu_c^{s(j)} - \varepsilon_{s(j),t} \label{eq:lowlevel_event2_two}
\end{align}
\end{subequations}
is true (similar to \eqref{eq:lowlevel_event1} and \eqref{eq:lowlevel_event2}), where $\varepsilon_{s(j),t}$ is now defined as
\begin{align}\label{eq:exploration_term_low_level_two}
\varepsilon_{s(j),t} = \sqrt{\frac{(K+1)\log t}{n_{s(j)}}}.
\end{align}
More specifically, if both \eqref{eq:lowlevel_event1_two} and \eqref{eq:lowlevel_event2_two} were false, then \eqref{eq:second_event} would be false, too. Thus, \eqref{eq:lemma_prob_bound_statement_ii} follows.
\end{proof}

\subsection*{Proof of Lemma \ref{lem:stopping_time_explicit}}

\begin{duplicateLemma}
The stopping time $\tau_{\mathcal{A}}$ is bounded as follows:
\begin{align*}
\frac{B}{\sum_{i\in a^\ast} \mu_c^i} &- c_2 - c_3 \log\left( c_1 + \frac{2B}{\sum_{i\in a^\ast}\mu_c^i} \right) \\
&\leq \tau_{\mathcal{A}} \leq \frac{2B}{\sum_{i\in a^\ast}}\mu_c^i + c_1,
\end{align*}
\end{duplicateLemma}

First, notice that the optimal algorithm $\mathcal{A}^\ast$ knows all bang-per-buck ratios and can simply pull those $K$ arms associated with the $K$ largest ratios, denoted with $a^\ast$. 
\begin{lemmaA}
The optimal expected payout of $\mathcal{A}^\ast$, $\mathbb{E}[U_{\mathcal{A}^\ast}]$, is bounded from above as follows:
\begin{align}\label{eq:optimal_payout}
\mathbb{E}[U_{\mathcal{A}^\ast}] \leq \frac{\sum_{i\in a^\ast}\mu_r^i}{\sum_{i\in a^\ast}\mu_c^i}(B+1)
\end{align}
\end{lemmaA}

\begin{proof}
This can be shown easily by induction. For the base case, consider $-1\leq B \leq 0$ and so \eqref{eq:optimal_payout} holds trivially. Now consider the budget $B^\prime > 0$. Then we have
\begin{align*}
U_{\mathcal{A}^\ast}(B^\prime) &= \sum_{i\in a^\ast} \mu_r^i + U_{\mathcal{A}^\ast}\left(B^\prime - \sum_{i\in a^\ast} \mu_r^i\right) \\
&\stackrel{\eqref{eq:optimal_payout}}{\leq} \sum_{i\in a^\ast} \mu_r^i + \frac{\sum_{i\in a^\ast}\mu_r^i}{\sum_{i\in a^\ast}\mu_c^i}\left(B^\prime - \sum_{i\in a^\ast} \mu_c^i + 1 \right) \\
&= \sum_{i\in a^\ast} \mu_r^i - \frac{\sum_{i\in a^\ast}\mu_r^i}{\sum_{i\in a^\ast}\mu_c^i}\sum_{i\in a^\ast}\mu_c^i \\
&\hspace*{0.4cm} + \frac{\sum_{i\in a^\ast}\mu_r^i}{\sum_{i\in a^\ast}\mu_c^i}(B^\prime + 1) = \frac{\sum_{i\in a^\ast}\mu_r^i}{\sum_{i\in a^\ast}\mu_c^i}(B^\prime + 1).
\end{align*}
\end{proof}

Now, let us denote the stopping time of the optimal algorithm as $\tau_{\mathcal{A}^\ast}(B)$. Since we know that $\mathcal{A}^\ast$ always selects the set of actions $a^\ast$ in each round, the stopping time is
\begin{align*}
\tau_{\mathcal{A}^\ast}(B) = \floor*{\frac{B}{\sum_{i\in a^\ast} \mu_c^i}}.
\end{align*}
Hence, we obtain the following inequality on $\tau_{\mathcal{A}^\ast}(B)$:
\begin{align}
\frac{B}{\sum_{i\in a^\ast} \mu_c^i} - 1 &\leq \tau_{\mathcal{A}^\ast}(B) \leq \frac{B}{\sum_{i\in a^\ast} \mu_c^i}. \label{eq:stopping_time_bound}
\end{align}
Lemma A\ref{lem:stopping_time_both_bounds} bounds stopping time $\tau_{\mathcal{A}}(B)$ of algorithm $\mathcal{A}$:
\begin{lemmaA}\label{lem:stopping_time_both_bounds}
The stopping time of algorithm $\mathcal{A}$ is bounded as follows:
\begin{align*}
&\frac{B - NK\left(\gamma \log \tau_{\mathcal{A}}(B) + \delta \right)}{\sum_{i\in a^\ast}\mu_c^i}-1 \\
&\hspace*{0.5cm}\leq \tau_{\mathcal{A}}(B) \leq \tau_{\mathcal{A}^\ast}(B) + \frac{N}{Kc_{\min}}\left(\gamma\log \tau_{\mathcal{A}}(B) + \delta\right).
\end{align*}
\end{lemmaA}

\begin{proof}
Let $0 \leq B^\ast \leq B$ denote the budget spent on pulling optimal arms from the set $a^\ast$ across all rounds $1, \ldots, \tau_{\mathcal{A}}(B)$. Similarly, let $B^-$ denote the budget spent on pulling non-optimal arms across those rounds. To obtain the upper bound on $\tau_{\mathcal{A}}(B)$, observe the following manipulations:
\begin{align}
\tau_{\mathcal{A}}(B) &\leq \tau_{\mathcal{A}^\ast}(B) + \tau_{\mathcal{A}}\left(\sum_{i\not\in a^\ast} n_{i,\tau(B)}c_{\max} \right) \nonumber\\
&\leq \tau_{\mathcal{A}^\ast}(B) + \tau_{\mathcal{A}}\left(\sum_{i=1}^N C_{i,\tau_{\mathcal{A}}(B)} \right)\nonumber\\
&\leq \tau_{\mathcal{A}^\ast}(B) + \frac{\sum_{i=1}^N C_{i,\tau(B)}}{Kc_{\min}}\label{eq:use_definition_counter}\\
&\stackrel{\eqref{eq:lower_bound_Cit}}{\leq} \tau_{\mathcal{A}^\ast}(B) + \frac{N}{Kc_{\min}}\left(\gamma\log \tau_{\mathcal{A}}(B) + \delta\right).\nonumber
\end{align}
In \eqref{eq:use_definition_counter}, we used the definition of the counters $C_{i,t}$ and the fact that the minimum payment per round is $Kc_{\min}$, from which $\tau_{\mathcal{A}(B)} \leq B/(Kc_{\min})$ follows for any $B$. To obtain the lower bound on $\tau_{\mathcal{A}}(B)$, observe the following:
\begin{align}
\tau_{\mathcal{A}}(B) &= \tau_{\mathcal{A}}(B^\ast + B^-)\nonumber\\
&\geq \tau_{\mathcal{A}^\ast}(B^\ast)\nonumber\\
&\geq \tau_{\mathcal{A}^\ast}\left( B - \sum_{i\not\in a^\ast}n_{i,\tau(B)}c_{\max} \right) \nonumber\\
&\geq \tau_{\mathcal{A}^\ast}\left( B - \sum_{i=1}^N C_{i,\tau_{\mathcal{A}}(B)} \right) \label{eq:use_definition_counter_2}\\
&\stackrel{\eqref{eq:lower_bound_Cit}}\geq \tau_{\mathcal{A}^\ast}\left( B - NK\left( \gamma \log\tau(B) + \delta \right) \right) \nonumber\\
&\stackrel{\eqref{eq:stopping_time_bound}}{\geq} \frac{B - NK\left(\gamma \log \tau_{\mathcal{A}}(B) + \delta \right)}{\sum_{i\in a^\ast}\mu_c^i} - 1.\nonumber
\end{align}
\eqref{eq:use_definition_counter_2} again uses the defintion of the counters $C_{i,t}$.
\end{proof}
Finally, to prove Lemma \ref{lem:stopping_time_explicit}, we need to remove the implicit relation of the bounds on $\tau_{\mathcal{A}}(B)$ presented in Lemma A\ref{lem:stopping_time_both_bounds}. For this purpose, we employ the inequality $\log(\phi) \leq \phi - 1$, which is valid for all $\phi > 0$. Letting $\phi = \frac{Kc_{\min}}{2N\gamma}\tau_{\mathcal{A}}(B)$, we obtain:
\begin{align}\label{eq:phi_identity}
\log \tau_{\mathcal{A}}(B) \leq \frac{Kc_{\min}}{2N\gamma}\tau_{\mathcal{A}}(B) + \log\left( \frac{2N\gamma}{Kc_{\min}} \right) - 1.
\end{align}
Substituting \eqref{eq:phi_identity} into the upper bound on $\tau_{\mathcal{A}}(B)$ in Lemma A\ref{lem:stopping_time_both_bounds} yields
\begin{align}
&\tau_{\mathcal{A}}(B) \leq \tau_{\mathcal{A}^\ast}(B) \nonumber \\
&+ \frac{N}{Kc_{\min}}\left[\gamma \left( \frac{Kc_{\min}}{2N\gamma}\tau_{\mathcal{A}}(B) + \log\left( \frac{2N\gamma}{Kc_{\min}} \right) - 1 \right) + \delta \right]\nonumber\\
&\hspace*{0.2cm}\stackrel{\eqref{eq:stopping_time_bound}}{\leq} \frac{B}{\sum_{i\in a^\ast}\mu_c^i} + \frac{\tau_{\mathcal{A}}(B)}{2} \nonumber \\
&\hspace*{0.8cm}+ \frac{N}{Kc_{\min}}\left[ \gamma\left( \log\left( \frac{2N\gamma}{Kc_{\min}} \right) - 1 \right) + \delta \right]\nonumber\\
&\hspace*{0.3cm}\leq \frac{2B}{\sum_{i\in a^\ast}\mu_c^i} \nonumber \\
&\hspace*{0.8cm}+ \underbrace{\frac{2N}{Kc_{\min}}\left[ \gamma\left( \log\left( \frac{2N\gamma}{Kc_{\min}} \right) - 1 \right) + \delta \right]}_{=:c_1}. \label{eq:stopping_time_explicit_upper_bound}
\end{align}
Next, taking the logarithm of \eqref{eq:stopping_time_explicit_upper_bound} and substituting into the lower bound on $\tau_{\mathcal{A}}(B)$ in Lemma A\ref{lem:stopping_time_both_bounds} results in the second part of the inequality in Lemma \ref{lem:stopping_time_explicit}, because
\begin{align}
\tau_{\mathcal{A}}(B) &\geq \frac{B}{\sum_{i\in a^\ast}\mu_c^i} - \underbrace{\left(\frac{NK\delta}{\sum_{i\in a^\ast}\mu_c^i} + 1\right)}_{=:c_2} \nonumber \\
&- \underbrace{\frac{NK\gamma}{\sum_{i\in a^\ast}\mu_c^i}}_{=:c_3}\log\left( \frac{2B}{\sum_{i\in a^\ast}\mu_c^i} + c_1 \right), \nonumber
\end{align}
where we again used \eqref{eq:stopping_time_bound}. This completes the proof.

\subsection*{Proof of Theorem \ref{thm:stochastic_case_sublinear_regret}}

\begin{duplicateTheorem}
There exist constants $c_1$, $c_2$, and $c_3$, which are functions of $N, K, c_{\min}, \mu_{i}, \mu_c$ only, such that Algorithm \texttt{UCB-MB} achieves expected regret
\begin{align*}
\mathcal{R}(B) \leq c_1 + c_2 \log(B + c_3) = O(NK^4 \log B)
\end{align*}
\end{duplicateTheorem}


\begin{proof}
The constants $c_1, c_2$, and $c_3$ were defined in the previous subsection and are repeated here for convenience:
\begin{align*}
c_1 &= \frac{2N}{Kc_{\min}}\left[\gamma\left(\log\left( \frac{2N\gamma}{Kc_{\min}} \right) - 1 \right) + \delta\right]\\
c_2 &= \left(\frac{NK\delta}{\sum_{i\in a^\ast}\mu_c^i} + 1\right)\\
c_3 &= \frac{NK\gamma}{\sum_{i\in a^\ast}\mu_c^i}
\end{align*}

Utilizing the definition of weak regret $\mathcal{R}_{\mathcal{A}}$ of a strategy $\mathcal{A}$ as the difference between the expected payout of the best strategy $\mathcal{A}^\ast$, which has knowledge of all bang-per-buck ratios, and the expected payout of $\mathcal{A}$, we obtain:
\begin{align}
\mathcal{R}_{\mathcal{A}} &= \mathbb{E}[G_{\mathcal{A}^\ast}] - \mathbb{E}[G_{\mathcal{A}}] \nonumber\\
&\stackrel{\eqref{eq:stopping_time_bound}}{\leq} \frac{\sum_{j\in a^\ast}\mu_r^j}{\sum_{j\in a^\ast}\mu_c^j}(B+1) - \mathbb{E}\left[ \sum_{t=1}^{\tau_{\mathcal{A}}(B)}\sum_{j\in a_t} \mu_r^j \right] \nonumber\\
&= \left[\frac{\sum_{j\in a^\ast}\mu_r^j}{\sum_{j\in a^\ast}\mu_c^j}(B+1) - \tau_{\mathcal{A}}(B)\sum_{j\in a^\ast}\mu_r^j \right] \nonumber \\
&\hspace*{0.4cm}+ \left[ \tau_{\mathcal{A}}(B)\sum_{j\in a^\ast}\mu_r^j- \mathbb{E}\left[ \sum_{t=1}^{\tau_{\mathcal{A}}(B)}\sum_{j\in a_t} \mu_r^j \right]\right] \nonumber\\
& \stackrel{\eqref{eq:increment_smallest_counter}}{\leq} \left[\frac{\sum_{j\in a^\ast}\mu_r^j}{\sum_{j\in a^\ast}\mu_c^j}(B+1) - \tau_{\mathcal{A}}(B)\sum_{j\in a^\ast}\mu_r^j \right] \nonumber\\
&\hspace*{0.4cm} + \sum_{i=1}^N C_{i,\tau_{\mathcal{A}}(B)} \Delta_{\max} \nonumber\\
&\leq \left[ \frac{\sum_{j\in a^\ast}\mu_r^j}{\sum_{j\in a^\ast}\mu_c^j}(B+1) - \sum_{j\in a^\ast}\mu_r^j \left( \frac{B}{\sum_{j\in a^\ast}\mu_c^j} \right.\right. \nonumber\\
&\hspace*{0.4cm}\left. \left. - c_2 - c_3\log\left( \frac{2}{\sum_{j\in a^\ast}\mu_c^j}B + c_1 \right) \right) \right] \label{eq:regret_stochastic_deriv_step_i} \\
&\quad + N\Delta_{\max} (\gamma \log\tau_{\mathcal{A}}(B) + \delta) \nonumber\\
&\leq \frac{\sum_{j\in a^\ast}\mu_r^j}{\sum_{j\in a^\ast}\mu_c^j} + \sum_{j\in a^\ast}\mu_r^j \left( c_2 + \right.\nonumber\\
&\hspace*{0.4cm}\left. c_3\log\left( \frac{2}{\sum_{j\in a^\ast}\mu_c^j}B + c_1 \right) \right) \label{eq:regret_stochastic_deriv_step_ii}\\
&\quad + N\Delta_{\max}\left( \gamma\log\left( \frac{2}{\sum_{j\in a^\ast}\mu_c^j}B + c_1 \right) +\delta \right) \nonumber \\
&= O(c_3 + N\gamma\log B) = O(NK^4 \log B) \label{eq:regret_stochastic_deriv_step_iii}
\end{align}
In \eqref{eq:regret_stochastic_deriv_step_i} and \eqref{eq:regret_stochastic_deriv_step_ii}, we used the explicit bounds on $\tau_{\mathcal{A}}(B)$ on $B$ derived in Lemma \ref{lem:stopping_time_explicit}. Lastly, in \eqref{eq:regret_stochastic_deriv_step_iii}, we used the definitions of the constants $c_3 = O(NK\gamma)$ and $\gamma = O(K^3)$. This completes the proof of Theorem \ref{thm:stochastic_case_sublinear_regret}.
\end{proof}
%
%

\section{Proofs for Adversarial Setting}\label{sec:proofs_adversarial_long}

\subsection*{Proof of Theorem \ref{thm:adversarial_case_bound}}
\begin{duplicateTheorem}
Algorithm \texttt{Exp3.M.B} achieves regret
\begin{align*}
\mathcal{R} \leq 2.63\sqrt{1 + \frac{B}{g c_{\min}}} \sqrt{gN\log(N/K)} + K,
\end{align*}
where $g$ is the maximal gain of the optimal algorithm. This bound is of order $O(\sqrt{BN\log(N/K)})$.
\end{duplicateTheorem}


\begin{proof}
Define $W_t = \sum_{i=1}^N w_i(t)$ and $\tilde{W}_t = \sum_{i=1}^N \tilde{w}_i(t)$. Then observe the following manipulations:
\begin{align}
&\frac{W_{t+1}}{W_t} = \sum_{i\in [N]\setminus \tilde{S}(t)} \frac{w_i(t)}{W_t}\exp\left( \frac{K\gamma}{N}(\hat{r}_i(t) - \hat{c}_i(t)) \right) \nonumber\\
&\hspace*{1.4cm}+ \sum_{i\in \tilde{S}(t)}\frac{w_i(t)}{W_t} \nonumber \\
&\hspace*{0.3cm}\leq \sum_{i\in [N]\setminus \tilde{S}(t)}\frac{w_i(t)}{W_t} \left[ 1 + \frac{K\gamma}{N}(\hat{r}_i(t) - \hat{c}_i(t)) \right.\nonumber\\
&\hspace*{0.7cm}\left.+ (e-2)\left( \frac{K\gamma}{N}(\hat{r}_i(t) - \hat{c}_i(t)) \right)^2 \right] + \sum_{i\in \tilde{S}(t)}\frac{w_i(t)}{W_t} \nonumber \\
&\hspace*{0.3cm}= 1 + \frac{\tilde{W}_t}{W_t}\sum_{i\in [N]\setminus \tilde{S}(t)}\frac{w_i(t)}{\tilde{W}_t} \left[ \frac{K\gamma}{N}(\hat{r}_i(t) - \hat{c}_i(t)) \right.\nonumber\\
&\hspace*{0.7cm}\left. + (e-2)\left( \frac{K\gamma}{N}(\hat{r}_i(t) - \hat{c}_i(t)) \right)^2 \right] \nonumber \\
&\hspace*{0.3cm}\leq 1 + \sum_{i\in [N]\setminus \tilde{S}(t)}\frac{p_i(t)/K - \gamma/N}{1-\gamma} \left[\frac{K\gamma}{N}(\hat{r}_i(t) - \hat{c}_i(t)) \right.\nonumber\\
&\hspace*{0.7cm}\left. + (e-2)\left( \frac{K\gamma}{N}(\hat{r}_i(t) - \hat{c}_i(t)) \right)^2 \right] \nonumber \\
&\hspace*{0.3cm}\leq 1 + \frac{\gamma}{(1-\gamma)N}\sum_{i\in [N]\setminus \tilde{S}(t)} p_i(t)(\hat{r}_i(t) - \hat{c}_i(t)) \nonumber\\
&\hspace*{0.7cm}+ \frac{(e-2)K\gamma^2}{(1-\gamma)N^2}\sum_{i\in [N]\setminus \tilde{S}(t)} p_i(t) (\hat{r}_i(t) - \hat{c}_i(t))^2 \nonumber \\
&\hspace*{0.3cm}\leq 1 + \frac{\gamma}{(1-\gamma)N}\sum_{i\in [N]\setminus \tilde{S}(t)} (r_i(t) - c_i(t)) \nonumber\\
&\hspace*{0.7cm}+ \frac{(e-2)K\gamma^2}{(1-\gamma)N^2}(1-c_{\min})\sum_{i\in [N]} (\hat{r}_i(t) - \hat{c}_i(t)). \nonumber 
\end{align}
In the above manipulations, we used the update rules of the weights and probabilities $p_i(t)$ defined in Algorithm \texttt{Exp3.M.B}. Further, we utilized the property $e^x \leq 1 + x + (e-2)x^2$ for $x = K\gamma (\hat{r}_i(t) - \hat{c}_i(t))/N <1$. In the last line, we exploit the definition of the estimated rewards $\hat{r}_i(t)$ and costs $\hat{c}_i(t)$. Next, since $e^x \geq 1 + x$ for $x\geq 0$, summing over $t=1, \ldots, T$, where $T = \max(\tau_{\mathcal{A}}(B), \tau_{\mathcal{A}^\ast}(B))$ yields
\begin{align}
&\log\left(\frac{W_{T+1}}{W_1}\right) \leq \frac{\gamma}{(1-\gamma)N}\sum_{t=1}^T \sum_{i\in a_t\setminus \tilde{S}(t)} (r_i(t) - c_i(t)) \nonumber \\
&\hspace*{0.25cm}+ \frac{(e-2)K\gamma^2}{(1-\gamma)N^2}(1-c_{\min})\sum_{t=1}^T\sum_{i\in [N]} (\hat{r}_i(t) - \hat{c}_i(t)). \label{eq:weight_update_sum}
\end{align}
Let $a^\ast$ denote the optimal action set for algorithm $\mathcal{A}^\ast$. Bounding $\log(W_{T+1}/W_1)$ from above yields
\begin{align}
&\log\left(\frac{W_{T+1}}{W_1}\right) \geq \log\left( \frac{\sum_{i\in a^\ast}w_i(T+1)}{W_1} \right) \nonumber\\
&\hspace*{0.4cm}\geq \log\left( \frac{K\left( \prod_{i\in a^\ast} w_i(T+1) \right)^{1/k}}{N} \right)\nonumber \\
&= \log\left(\frac{K}{N}\right) \nonumber\\
&\hspace*{0.4cm} + \frac{1}{K}\log\left( \prod_{i\in a^\ast}\prod_{t\in [T]:i\not\in \tilde{S}(t)}\exp\left(\frac{K\gamma}{N}(\hat{r}_i(t) - \hat{c}_i(t)) \right)\right) \nonumber \\
 &= \log\left(\frac{K}{N}\right) + \frac{1}{K}\sum_{i\in a^\ast}\sum_{t\in [T]:i\not\in \tilde{S}(t)} \frac{K\gamma}{N}\left(\hat{r}_i(t) - \hat{c}_i(t)\right). \label{eq:upper_bound_W}
\end{align}
Combining \eqref{eq:weight_update_sum} and \eqref{eq:upper_bound_W} yields
\begin{align}
&\frac{N}{\gamma}\log\left(\frac{K}{N}\right) + \sum_{i\in a^\ast}\sum_{t:i\not\in\tilde{S}} \left(\hat{r}_i(t) - \hat{c}_i(t)\right) \nonumber\\
&\hspace*{0.3cm}\leq \frac{1}{1-\gamma}\sum_{t=1}^T \sum_{i\in a_t\setminus \tilde{S}(t)} (r_i(t) - c_i(t)) \label{eq:adversarial_bound_step1}\\
&\hspace*{0.7cm}+ \frac{(e-2)\gamma K(1-c_{\min})}{N(1-\gamma)}\sum_{t=1}^T \sum_{i\in [N]} \left(\hat{r}_i(t) - \hat{c}_i(t)\right). \nonumber
\end{align}
Taking the expectations of $\hat{r}_i(t)$ and $\hat{c}_i(t)$ and adding the term $\frac{1}{1-\gamma}\sum_{t=1}^T \sum_{i\in \tilde{S}(t)} (r_i(t) - c_i(t))$, which is bounded from below by $\sum_{i\in a^\ast} \sum_{t:i\in\tilde{S}(t)}(r_i(t) - c_i(t))$, to both sides of \eqref{eq:adversarial_bound_step1} gives us
\begin{align}
&\frac{N}{\gamma}\log\left(\frac{K}{N}\right) + \sum_{i\in a^\ast}\sum_{t=1}^T \left(r_i(t) - c_i(t)\right) \nonumber \\
&\hspace*{0.3cm}\leq \frac{1}{1-\gamma}\sum_{t=1}^T \sum_{i\in a_t} (r_i(t) - c_i(t)) \label{eq:adversarial_bound_step2}\\
&\hspace*{0.7cm}+ \frac{(e-2)\gamma K(1-c_{\min})}{N(1-\gamma)}\sum_{t=1}^T \sum_{i\in [N]} \left(r_i(t) - c_i(t)\right). \nonumber
\end{align}
Since $T = \max(\tau_{\mathcal{A}}(B), \tau_{\mathcal{A}^\ast}(B))$ and due to the fact that algorithm $\mathcal{A}$ terminates after $\tau_{\mathcal{A}}(B)$ rounds, \eqref{eq:adversarial_bound_step2} becomes
\begin{align}
&\frac{N}{\gamma}\log\left(\frac{K}{N}\right) + \sum_{i\in a^\ast}\sum_{t=1}^{\tau_{\mathcal{A}^\ast}(B)} \left(r_i(t) - c_i(t)\right) \nonumber\\
&\hspace*{0.3cm}\leq \frac{1}{1-\gamma}\sum_{t=1}^{\tau_{\mathcal{A}}(B)} \sum_{i\in a_t} (r_i(t) - c_i(t)) \label{eq:adversarial_bound_step3}\\
&\hspace*{0.7cm}+ \frac{(e-2)\gamma K(1-c_{\min})}{N(1-\gamma)}\sum_{t=1}^{T} \sum_{i\in [N]} \left(r_i(t) - c_i(t)\right) \nonumber
\end{align}
We now bound the time-dependent terms in \eqref{eq:adversarial_bound_step3} separately:
\begin{align*}
\sum_{i\in a^\ast}\sum_{t=1}^{\tau_{\mathcal{A}^\ast}(B)} \left(r_i(t) - c_i(t)\right) &\geq G_{\max} - B \\
\sum_{t=1}^{\tau_{\mathcal{A}}(B)} \sum_{i\in a_t} (r_i(t) - c_i(t)) &\leq \left( G_{\texttt{Exp3.M.B}} - (B - Kc_{\max}) \right)
\end{align*}
and 
\begin{align}
\sum_{t=1}^{T} \sum_{i\in [N]} c_i(t) &\geq \sum_{t=1}^{\tau_{\mathcal{A}}(B)} \sum_{i\in [N]} c_i(t) \geq B - Kc_{\max} \nonumber\\
\sum_{t=1}^{T} \sum_{i\in [N]} r_i(t) &= \sum_{i\in [N]} \sum_{t=1}^{\tau_{\mathcal{A}^\ast}(B)} r_i(t) + \mathds{1}\left( \tau_{\mathcal{A}}(B) > \tau_{\mathcal{A}^\ast}(B) \right)\nonumber\\
&\hspace*{0.4cm}\times\sum_{i\in [N]}\sum_{t=\tau_{\mathcal{A}^\ast}(B)+1}^{\tau_{\mathcal{A}}(B)} r_i(t)\nonumber \\
&\leq \frac{N}{K}G_{\max} + \frac{NB(1-c_{\min})}{Kc_{\min}}.\label{eq:max_reward_excess}
\end{align}
In \eqref{eq:max_reward_excess}, we used the upper bound
\begin{align*}
&\mathds{1}\left( \tau_{\mathcal{A}}(B) > \tau_{\mathcal{A}^\ast}(B) \right)\cdot\sum_{i\in [N]}\sum_{t=\tau_{\mathcal{A}^\ast}(B)+1}^{\tau_{\mathcal{A}}(B)} r_i(t) \nonumber\\
&\hspace*{0.3cm}\leq |\tau_{\mathcal{A}}(B) - \tau_{\mathcal{A}^\ast}(B)| Nr_{\max} \\
&\hspace*{0.3cm}\leq \frac{B - \frac{B}{Kc_{\max}}Kc_{\min}}{Kc_{\min}}Nc_{\max} =  \frac{NB(1-c_{\min})}{Kc_{\min}}.
\end{align*}
With these bounds, \eqref{eq:adversarial_bound_step3} becomes
\begin{align}
&G_{\max} - G_{\texttt{Exp3.M.B}} \nonumber\\
&\hspace*{0.3cm}\leq \frac{N}{\gamma} \log\left(\frac{N}{K}\right) + \gamma G_{\max} \left(1 + (e-2)\gamma(1-c_{\min})\right) \nonumber \\
&\hspace*{0.7cm}+ \gamma B\left(\frac{(e-2)(1-c_{\min})^2}{c_{\min}} - 1\right) + K \nonumber \\
&\hspace*{0.3cm}\leq \frac{N}{\gamma} \log\left(\frac{N}{K}\right) + \gamma G_{\max}(e-1) + \frac{\gamma B(e-1)}{c_{\min}} + K. \nonumber
\end{align}
If an upper bound $g$ on $G_{\max}$ exists, i.e. $g \geq G_{\max}$, then $\gamma$ can be tuned by choosing $\gamma = \min\left(1, \sqrt{\frac{N\log(N/K)}{g(e-1)(1 + \frac{B}{gc_{\min}})}} \right)$, which gives us
\begin{align*}
&\mathcal{R} = G_{\max} - G_{\texttt{Exp3.M.B}} \\
&\hspace*{0.3cm}\leq K + 2\sqrt{e-1}\sqrt{1 + \frac{B}{gc_{\min}}} \sqrt{g N \log(N/K)} \nonumber\\
&\hspace*{0.3cm}< K + 2.63 \sqrt{1 + \frac{B}{g c_{\min}}}\sqrt{gN\log(N/K)},
\end{align*}
as stated in Theorem \ref{thm:adversarial_case_bound}.
\end{proof}

\subsection*{Proof of Proposition \ref{prop:upper_bound_no_g}}

\begin{duplicateProposition}
For the multi-play case with budget, the regret is upper bounded by
\begin{align}
\mathcal{R} \leq 8\left[(e-1)-(e-2)c_{\min}\right]\frac{N}{K} + 2N\log\frac{N}{K}+ K +\nonumber\\\
8 \sqrt{\left[(e-1)-(e-2)c_{\min}\right](G_{\max}-B+K)N\log(N/K)}\nonumber
\end{align}
\end{duplicateProposition}


\subsubsection*{Proof of Lemma \ref{lem:regret_suffered_selected_epoch_with_budget}}
\begin{duplicateLemma}
For any subset $a \in \mathcal{S}$ of $K$ unique elements from $[N]$, $1\leq K \leq N$:
\begin{align}
&\sum_{t=S_r}^{T_r} \sum_{i\in a_t} (r_i(t)-c_i(t)) \geq \sum_{i\in a} \sum_{t=S_r}^{T_r} (\hat{r}_j(t) - \hat{c}_j(t)) \\
&\hspace*{0.5cm} - 2\sqrt{(e-1)-(e-2)c_{\min}}\sqrt{g_r N\log(N/K)}, \nonumber
\end{align}
where $S_r$ and $T_r$ denote the first and last time step at epoch $r$, respectively. 
\end{duplicateLemma}


\begin{proof}
Using the update rule for weights \eqref{eq:weight_update_sum} in place of the original update rule in Algorithm \texttt{Exp3.M.B}, we obtain the following inequality from \eqref{eq:adversarial_bound_step1} in the proof of Theorem \ref{thm:adversarial_case_bound}:
\begin{equation}\label{eq:adversarial_bound_step_modified_budget}
\begin{aligned}
&\sum_{t=T_r}^{S_r} \sum_{i\in a_t} (r_i(t) - c_i(t)) \geq \\
&\hspace*{0.4cm}(1-\gamma_r)\left[ \sum_{i\in a} \sum_{t=T_r}^{S_r} (\hat{r}_i(t) - \hat{c}_i(t)) + \frac{N}{\gamma_r}\log\frac{K}{N}\right]\\
&\quad -(1-\gamma_r)\frac{(e-2)\gamma_r K}{N(1-\gamma_r)}\sum_{t=T_r}^{S_r} \sum_{i\in [N]} (\hat{r}_i(t) - \hat{c}_i(t)), 
\end{aligned}
\end{equation}
where $a$ denotes any subset of $[N]$ of size $K$. According to the termination criterion of Algorithm \texttt{Exp3.1.M.B}, for each epoch $r$ we have $\sum_{i\in a} (\hat{G}_i(T_r) - \hat{L}_i(T_r)) \leq g_r - \frac{N(1-c_{\min})}{K\gamma_r}$ for all $a\in\mathcal{S}$ and therefore
\begin{align*}
\sum_{i\in a}\hat{G}_i(T_r + 1) \leq  g_r - \frac{N(1-c_{\min})}{K\gamma_r} + \frac{N(1-c_{\min})}{K\gamma_r} = g_r.
\end{align*}
Combining this equation with \eqref{eq:adversarial_bound_step_modified_budget} yields  \eqref{eq:regret_suffered_selected_epoch_with_budget}, as stated in Lemma \ref{lem:regret_suffered_selected_epoch_with_budget}.
\end{proof}

\subsubsection*{Proof of Lemma \ref{lem:num_epochs_exp31m_with_budget}}
\begin{duplicateLemma}
The total number of epochs $R$ is bounded above by
\begin{align}
2^{R-1} \leq \frac{N(1-c_{\min})}{Kc} + \sqrt{\frac{\hat{G}_{\max} - \hat{L}_{\max}}{c}} + \frac{1}{2},
\end{align}
where $c = \frac{N\log(N/K)}{(e-1) - (e-2)c_{\min}}$.
\end{duplicateLemma}


\begin{proof}
Observe that
\begin{align}
&\hat{G}_{\max}(T+1) - \hat{L}_{\max}(T+1) \nonumber\\
&\hspace{0.2cm}\geq \hat{G}_{\max}(T_{R-1}+1) - \hat{L}_{\max}(T_{R-1}+1) \nonumber\\
&\hspace{0.2cm}\geq g_{R-1} + \frac{N(1-c_{\min})}{K\gamma_{R-1}} \label{eq:num_epochs_bounded}\\
&\hspace{0.2cm}= 4^{R-1} c - 2^{R-1}\frac{N(1-c_{\min})}{K} =: cz^2 - \frac{N(1-c_{\min})}{K}z, \nonumber
\end{align}
where $z = 2^{R-1}$. Clearly, \eqref{eq:num_epochs_bounded} is increasing for $z > N(1-c_{\min})/(2Kc)$. Now, if \eqref{eq:num_epochs_exp31m_with_budget_upperbound} were false, then $z > \frac{N(1-c_{\min})}{Kc} + \sqrt{(\hat{G}_{\max}-\hat{L}_{\max})/c} > N(1-c_{\min})/(2Kc)$ would be true, and as a consequence,
\begin{align*}
&cz^2 - \frac{N(1-c_{\min})}{K}z \nonumber\\
&\hspace*{0.3cm} >c\left( \frac{N(1-c_{\min})}{Kc} + \sqrt{\frac{\hat{G}_{\max}-\hat{L}_{\max}}{c}} \right)^2 \\
&\hspace*{0.7cm}- \frac{N(1-c_{\min})}{K}\left( \frac{N(1-c_{\min})}{Kc} + \sqrt{\hat{G}_{\max}/c} \right)\\
&\hspace*{0.1cm}= \frac{N(1-c_{\min})}{K}\sqrt{\frac{\hat{G}_{\max}-\hat{L}_{\max}}{c}} + \hat{G}_{\max}-\hat{L}_{\max}
\end{align*}
which contradicts \eqref{eq:num_epochs_bounded}. 
\end{proof}

To prove Proposition \ref{prop:upper_bound_no_g}, we put together the results from Lemmas \ref{lem:regret_suffered_selected_epoch_with_budget} and \ref{lem:num_epochs_exp31m_with_budget}. Then we obtain
\begin{align*}
&\sum_{t=1}^{\tau_{\mathcal{A}}(B)}\sum_{i \in a_t}(r_i(t) - c_i(t)) \geq \max_{a\in\mathcal{S}}\left(\sum_{t=1}^{\tau_{\mathcal{A}^\ast}(B)}\sum_{i\in a}(\hat{r}_i(t)-\hat{c}_i(t)) \right.\\
&\left.\quad -2\sqrt{(e-1)-(e-2)c_{\min}}\sum_{r=0}^R \sqrt{g_r N\log(N/K)}\right),
\end{align*}
as we showed in Lemma \ref{lem:regret_suffered_selected_epoch_with_budget} that this bounds holds for \textit{any} subset of arms. Continuing the above equations, we observe
\begin{align}
&\sum_{t=1}^{\tau_{\mathcal{A}}(B)}\sum_{i \in a_t}(r_i(t) - c_i(t)) \nonumber \\
&\hspace*{0.3cm} \geq \hat{G}_{\max} - \hat{L}_{\max} - 2N\log(N/K)\sum_{r=0}^R 2^r \nonumber\\
&\hspace*{0.3cm} \geq \hat{G}_{\max} - \hat{L}_{\max} + 2N\log(N/K) \nonumber\\
&\hspace*{0.7cm} - 8N\log(N/K)\left( \frac{N(1-c_{\min})}{Kc} + \frac{\hat{G}_{\max} - \hat{L}_{\max}}{c} + \frac{1}{2} \right) \nonumber\\
&\hspace{0.3cm}\geq \hat{G}_{\max} - \hat{L}_{\max} - 2N\log\frac{N}{K} \nonumber \\
&\hspace*{0.7cm} - 8((e-1) - (e-2)c_{\min})\frac{N}{K}\label{eq:GmaxMinusLmax}\\
&\hspace*{0.7cm}-8\sqrt{((e-1) - (e-2)c_{\min})N\log\frac{N}{K}(\hat{G}_{\max} - \hat{L}_{\max})}.\nonumber
\end{align}
On the other hand, we have
\begin{align}\label{eq:GmaxMinusLmax2}
\sum_{t=1}^{\tau_{\mathcal{A}}(B)}\sum_{i \in a_t} (r_i(t) - c_i(t)) \leq G_{\texttt{Exp3.1.M.B}} - (B-K)
\end{align}
Simply combining \eqref{eq:GmaxMinusLmax} and \eqref{eq:GmaxMinusLmax2} yields
\begin{align}
&G_{\texttt{Exp3.1.M.B}} \nonumber\\
&\hspace{0.3cm} \geq B - K + \hat{G}_{\max} - \hat{L}_{\max} - 2N\log\frac{N}{K} \nonumber \\
&\hspace*{0.7cm} - 8((e-1) - (e-2)c_{\min})\frac{N}{K}\nonumber\\
&\hspace*{0.7cm}-8\sqrt{((e-1) - (e-2)c_{\min})N\log\frac{N}{K}(\hat{G}_{\max} - \hat{L}_{\max})} \nonumber\\
&\hspace*{0.3cm} =: f(\hat{G}_{\max} - \hat{L}_{\max}).\label{eq:f(GmaxLmax)}
\end{align}
and it can be shown that $f(\hat{G}_{\max} - \hat{L}_{\max})$ is convex. Thus, taking the expectation of \eqref{eq:f(GmaxLmax)} and utilizing Jensen's inequality gives
\begin{align*}
\mathbb{E}[G_{\texttt{Exp3.1.M.B}}] &\geq \mathbb{E}[f(\hat{G}_{\max} - \hat{L}_{\max})]\\
&\geq f(\mathbb{E}[\hat{G}_{\max} - \hat{L}_{\max}]).
\end{align*}
Further, we notice
\begin{align*}
\mathbb{E}[\hat{G}_{\max} - \hat{L}_{\max}] &= \mathbb{E}\left[\max_{a\in\mathcal{S}}\sum_{i\in a} \hat{G}_i - \hat{L}_i\right] \\
&\geq \max_{a\in \mathcal{S}}\mathbb{E}\left[ \sum_{i\in a}\hat{G}_i - \hat{L}_i \right]\\
&=\max_{a\in\mathcal{S}} \sum_{i\in a}\sum_{t=1}^{\tau_{\mathcal{A}}(B)}(r_i(t) - c_i(t)) \\
&\geq G_{\max} - (B - K).
\end{align*}
These results, together with the elementary fact $\mathbb{E}[\hat{L}_{\max}] \leq B$, yield the claim in Proposition \ref{prop:upper_bound_no_g}.

\subsection*{Proof of Theorem \ref{thm:lower_bound_multiple_play}}
\begin{duplicateTheorem}
For $1\leq K \leq N$, the weak regret $\mathcal{R}$ is lower bounded as follows:
\begin{align}
\mathcal{R} \geq \varepsilon\left( B - \frac{BK}{N} - 2Bc_{\min}^{-3/2}\varepsilon\sqrt{\frac{BK\log(4/3)}{N}} \right),
\end{align}
where $\varepsilon \in (0, 1/4)$. Choosing $\varepsilon$ as
\begin{align*}
\varepsilon = \min\left(\frac{1}{4},~\frac{(1-K/N)c_{\min}^{3/2}}{4\sqrt{\log(4/3)}}\sqrt{\frac{N}{BK}}\right)
\end{align*}
yields the bound
\begin{align}
\mathcal{R} \geq \min\left(\frac{c_{\min}^{3/2}(1-K/N)^2}{8\sqrt{\log(4/3)}}\sqrt{\frac{NB}{K}} ,~\frac{B(1-K/N)}{8}\right)
\end{align}
\end{duplicateTheorem}


\subsubsection*{Proof of Lemma \ref{lem:function_on_reward_cost_sequences}}
As mentioned in the main text, we use the auxiliary Lemma \ref{lem:function_on_reward_cost_sequences}:

\begin{duplicateLemma}
Let $f : \lbrace \lbrace 0, 1\rbrace, \lbrace c_{\min}, 1\rbrace\rbrace^{\tau_{\max}}\to [0, M]$ be any function defined on reward and cost sequences $\lbrace \mathbf{r}, \mathbf{c}\rbrace$ of length less than or equal $\tau_{\max} = \frac{B}{Kc_{\min}}$. Then, for the best action set $a^\ast$:
\begin{align}
&\mathbb{E}_{a^\ast}\left[ f(\mathbf{r}, \mathbf{c}) \right] \\
&\hspace*{0.5cm}\leq \mathbb{E}_{u}[f(\mathbf{r}, \mathbf{c})] + \frac{Bc_{\min}^{-3/2}}{2}\sqrt{-\mathbb{E}_{u}[N_{a^\ast}]\log(1-4\varepsilon^2)}. \nonumber
\end{align}
\end{duplicateLemma}


\begin{proof}
Let $\mathbf{r}_t$ and $\mathbf{c}_t$ denote the vector of rewards and costs observed at time $t$, respectively. Similarly, let $\mathbf{r}^t$ and $\mathbf{c}^t$ denote all such reward and cost vectors observed up to time $t$. $\mathbb{P}_{u}(\cdot)$ or $\mathbb{P}_{a^\ast}(\cdot)$ are probability measures of a random variable with respect to the uniform assignment of costs $\lbrace c_{\min}, 1 \rbrace$ and rewards $\lbrace 0, 1 \rbrace$ to arms or conditional on $a^\ast$ being the best subset of arms, respectively. With this notation, we have
\begin{align}
&\mathbb{E}_{a^\ast}[f(\mathbf{r}, \mathbf{c})] - \mathbb{E}_{u}[f(\mathbf{r}, \mathbf{c})] \nonumber\\
&\hspace*{0.4cm}= \sum_{\mathbf{r}, \mathbf{c}} f(\mathbf{r}, \mathbf{c})(\mathbb{P}_{a^\ast}(\mathbf{r}, \mathbf{c}) - \mathbb{P}_{u}(\mathbf{r}, \mathbf{c}))\nonumber\\
&\hspace*{0.4cm}\leq \frac{B}{c_{\min}}\sum_{(\mathbf{r}, \mathbf{c}): \mathbb{P}_{a^\ast}(\mathbf{r}, \mathbf{c}) \geq \mathbb{P}_{u}(\mathbf{r}, \mathbf{c})} (\mathbb{P}_{a^\ast}(\mathbf{r}, \mathbf{c}) - \mathbb{P}_{u}(\mathbf{r}, \mathbf{c}))\nonumber\\
&\hspace*{0.4cm}\leq \frac{B}{2c_{\min}}\Vert \mathbb{P}_{a^\ast} - \mathbb{P}_{u}\Vert_1, \label{eq:Ei_Eunif_diff}
\end{align}
where $\Vert \mathbb{P}_{a^\ast} - \mathbb{P}_{u}\Vert_1 = \sum_{(\mathbf{r}, \mathbf{c})}|\mathbb{P}_{a^\ast}(\mathbf{r}, \mathbf{c}) - \mathbb{P}_{u}(\mathbf{r}, \mathbf{c})|$. Letting $\text{Bern}(p)$ denote a Bernoulli distribution with parameter $p$, we obtain, using Pinsker's Inequality
\begin{align}
\Vert \mathbb{P}_{a^\ast} - \mathbb{P}_{u}\Vert_1^2 \leq 2\log 2\cdot\mathrm{KL}(\mathbb{P}_{u} ~\Vert~ \mathbb{P}_{a^\ast}),\label{eq:variational_distance_KL}
\end{align}
the following result:
\begin{align}
&\mathrm{KL}(\mathbb{P}_{u} ~\Vert~ \mathbb{P}_{a^\ast}) \nonumber\\
&\hspace*{0.2cm}= \sum_{t=1}^{\floor*{ \frac{B}{Kc_{\min}}}}\mathds{1}\left( \sum_{\tau=1}^{t-1}\mathbf{1}\cdot \mathbf{c}_{\tau} \leq B - c_{\min} \right) \mathds{1}\left( \sum_{\tau=1}^{t} \mathbf{1}\cdot \mathbf{c}_{\tau} \leq B \right)\nonumber\\
&\hspace*{0.4cm}\times\mathrm{KL}\left( \mathbb{P}_{u} \left(\mathbf{c}_t, \mathbf{r}_t~|~\mathbf{c}^{t-1}, \mathbf{r}^{t-1}\right) ~\Vert~ \mathbb{P}_{a^\ast}\left(\mathbf{c}_t, \mathbf{r}_t~|~\mathbf{c}^{t-1}, \mathbf{r}^{t-1}\right) \right) \nonumber\\
&\hspace*{0.2cm}\leq \sum_{t=1}^{\floor*{ \frac{B}{Kc_{\min}}}}\mathbb{P}_{u}(a_t \neq a^\ast)\mathrm{KL}\left(\text{Bern}(1/2)~\Vert~\text{Bern}(1/2)\right) \nonumber\\
&\hspace*{1.6cm}+ \mathbb{P}_{u}(a_t = a^\ast)\mathrm{KL}\left(\text{Bern}(1/2)~\Vert~\text{Bern}(\varepsilon + 1/2)\right)\nonumber\\
&\hspace*{0.2cm}= \sum_{t=1}^{\floor*{ \frac{B}{Kc_{\min}}}}\mathbb{P}_{u}(a_t = a^\ast)\left(-\frac{1}{2}\log_{2}(1-4\varepsilon^2)\right)\nonumber\\
&\hspace*{0.2cm}= \frac{1 + c_{\min}}{2c_{\min}}\mathbb{E}_{u}[N_{a^\ast}]\left(-\frac{1}{2}\log_2(1-4\varepsilon^2)\right) \label{eq:KL_stopping_time}\\
&\hspace*{0.2cm}\leq \frac{1}{c_{\min}}\mathbb{E}_{u}[N_{a^\ast}]\left(-\frac{1}{2}\log_2(1-4\varepsilon^2)\right),\label{eq:KL_Punif_Pi}
\end{align}
where in \eqref{eq:KL_Punif_Pi} we used $c_{\min} \leq 1$. \eqref{eq:KL_stopping_time} uses the expected stopping time under uniform assignment $\mathbb{E}_{u}\left[ \tau(B)\right] = \floor*{2BK^{-1}/(c_{\min} + 1)}$ to obtain
\begin{align*}
&\sum_{t=1}^{\floor*{ \frac{B}{Kc_{\min}}}}\mathbb{P}_{u}(a_t = a^\ast)\nonumber\\
&= \sum_{t=1}^{\floor*{ \frac{2B}{K(1+c_{\min})}}}\mathbb{P}_{u}(a_t = a^\ast) + \sum_{t=\ceil*{ \frac{2B}{K(1+c_{\min})}}}^{\floor*{\frac{B}{Kc_{\min}}}}\mathbb{P}_{u}(a_t = a^\ast)\nonumber\\
&\hspace*{0.4cm}= \frac{B/(Kc_{\min})}{2BK^{-1}/(1+c_{\min})}\mathbb{E}_{u}(N_{a^\ast}) \leq \frac{1}{c_{\min}}\mathbb{E}_{u}(N_{a^\ast}).
\end{align*}
Substituting \eqref{eq:variational_distance_KL} and \eqref{eq:KL_Punif_Pi} into \eqref{eq:Ei_Eunif_diff} and utilizing $\log_2(x) = \log x / \log 2$ for $x > 0$ yields the statement in the lemma.
\end{proof}

To finalize the proof of Theorem \ref{thm:lower_bound_multiple_play}, notice that there exist $N\choose K$ possible combinations of arms of size $K$. Borrowing notation from \cite{Uchiya:2010aa}, let $\mathbf{C}([N], K)$ denote the set of all such subsets. Now, let $\mathbb{E}_{\ast}[\cdot]$ denote the expected value with respect to the uniform assignment of ``good'' arms. With this notation, observe
\begin{align*}
\mathbb{E}_{\ast}[G_{\max}] &= \left(\frac{1}{2} + \varepsilon\right)K \mathbb{E}_{\ast}[\tau_{\mathcal{A}}(B)]\\
\mathbb{E}_{a^\ast}[G_{\mathcal{A}}] &= \frac{1}{2}K \mathbb{E}_{a^\ast}[\tau_{\mathcal{A}}(B)] + \varepsilon \mathbb{E}_{a^\ast}[N_{a^\ast}] \\
\mathbb{E}_{\ast}[G_{\mathcal{A}}] &= \frac{1}{{N\choose K}} \sum_{a^\ast\in \mathbf{C}([N], K)} \mathbb{E}_{a^\ast}[G_{\mathcal{A}}] \\
&= \frac{1}{2}K \mathbb{E}_{\ast}[\tau_{\mathcal{A}}(B)] + \frac{\varepsilon}{{N\choose K}} \sum_{a^\ast\in \mathbf{C}([N], K)} \mathbb{E}_{a^\ast}[N_{a^\ast}].
\end{align*}
Therefore, we have
\begin{align}
&\mathbb{E}_{\ast}[G_{\max} - G_{\mathcal{A}}] \nonumber\\
&\hspace*{0.3cm}\geq \varepsilon K \mathbb{E}_{\ast}[\tau_{\mathcal{A}}(B)] - \frac{\varepsilon}{{N \choose K}} \sum_{a^\ast\in \mathbf{C}([N], K)} \left( \mathbb{E}_{u}[N_{a^\ast}] \right. \nonumber\\
&\hspace*{0.7cm}\left. + \frac{B}{2c_{\min}^{3/2}} \sqrt{-\mathbb{E}_{u}[N_{a^\ast}]\log(1-4\varepsilon^2)} \right) \nonumber\\
&\hspace*{0.3cm}\geq \varepsilon K \mathbb{E}_{u}[\tau_{\mathcal{A}}(B)] - \frac{\varepsilon}{{N\choose K}}{N\choose K}\mathbb{E}_{u}[\tau_{\mathcal{A}}(B)]\frac{K}{N}K \nonumber\\
&\hspace*{0.6cm} - \frac{\varepsilon}{{N\choose K}}\sum_{a^\ast\in \mathbf{C}([N], K)}\frac{B}{2c_{\min}^{3/2}} \sqrt{-\mathbb{E}_{u}[N_{a^\ast}]\log(1-4\varepsilon^2)} \nonumber\\
&\hspace*{0.3cm}= \varepsilon K\left(1-\frac{K}{N}\right)\mathbb{E}_{u}[\tau_{\mathcal{A}}(B)] \label{eq:multiple_play_proof_jensen}\\
&\hspace*{0.7cm}- \frac{\varepsilon Bc_{\min}^{-3/2}}{2{N\choose K}}\sqrt{-{N\choose K}{N-1 \choose K-1} \mathbb{E}_{u}[N_{a^\ast}]\log(1-4\varepsilon^2)}\nonumber\\
&\hspace*{0.3cm}\geq \varepsilon B\left(1-\frac{K}{N}\right) - \frac{2\varepsilon B}{c_{\min}^{3/2}}\sqrt{\frac{BK}{N}\log(4/3)}.\label{eq:multiple_play_proof_tau}
\end{align}
In \eqref{eq:multiple_play_proof_jensen}, we used Jensen's inequality and the fact that
\begin{align*}
\sum_{a^\ast\in \mathbf{C}([N], K)}\mathbb{E}_{u}[N_{a^\ast}] = {N\choose K}\mathbb{E}_{u}[\tau_{\mathcal{A}}(B)]\frac{K}{N}K
\end{align*}
In \eqref{eq:multiple_play_proof_tau}, we utilized $B/K \leq \mathbb{E}_{u}[\tau_{\mathcal{A}}(B)] \leq B/(2K)$ and $-\log(1-4\varepsilon^2) \leq 16\log(4/3)\varepsilon^2$. Finally, to prove \eqref{eq:lower_bound_regret_multiple_play}, we tune $\varepsilon$ as follows:
\begin{align}\label{eq:multiple_play_lower_bound_eps_choice}
\varepsilon = \min\left(\frac{1}{4},~\frac{c_{\min}^{3/2}}{4\log(4/3)}(1-K/N)\sqrt{\frac{N}{BK}}\right).
\end{align}
Plugging \eqref{eq:multiple_play_lower_bound_eps_choice} back into \eqref{eq:multiple_play_proof_tau} completes the proof.

\subsection*{Proof of Theorem \ref{thm:multiple_play_fixed_round_high_probability}}

\begin{duplicateTheorem}
For the multiple play algorithm ($1\leq K \leq N$) and a fixed number of rounds $T$, the following bound on the regret holds with probability at least $1-\delta$:
\begin{align}
\mathcal{R} &= G_{\max} - G_{\texttt{Exp3.P.M}} \nonumber\\
&\hspace*{0.4cm}\leq 2\sqrt{5}\sqrt{NKT\log(N/K)} + 8\frac{N-K}{N-1}\log\left( \frac{NT}{\delta} \right) \nonumber\\
&\hspace*{0.8cm}+ 2(1+K^2)\sqrt{NT\frac{N-K}{N-1}\log\left( \frac{NT}{\delta} \right)}
\end{align}
\end{duplicateTheorem}


\subsubsection*{Proof of Lemma \ref{lem:confidence_level_exp3p}}
\begin{duplicateLemma}
For $2\sqrt{\frac{N-K}{N-1} \log\left( \frac{NT}{\delta} \right)} \leq \alpha \leq 2\sqrt{NT}$,
\begin{align}
&\mathbb{P}\left( \hat{U}^\ast > G_{\max} \right) \nonumber\\
&\hspace*{0.4cm}\geq\mathbb{P}\left( \bigcap_{a\subset\mathcal{S}} \sum_{i\in a} \hat{G}_i + \alpha \hat{\sigma}_i > \sum_{i\in a} G_i \right) \geq 1-\delta,
\end{align}
where $S\subset\mathcal{S}$ denotes an arbitrary subset of $1 \leq K < N$ unique elements from $[N]$. $\hat{U}^\ast$ denotes the upper confidence bound for the optimal gain.
\end{duplicateLemma}


\begin{proof}
Since 
\begin{align*}
&\mathbb{P}\left( \bigcap_{a\in\mathcal{S}} \sum_{i\in a} \hat{G}_i + \alpha \hat{\sigma}_i > \sum_{i\in a} G_i \right) \geq\\
& \mathbb{P}\left( \bigcap_{i\in [N]}  \hat{G}_i + \alpha \hat{\sigma}_i >  G_i \right) = 1 - \mathbb{P}\left( \bigcup_{i\in [N]} \hat{G}_i + \alpha \hat{\sigma}_i \leq  G_i \right),
\end{align*}
it suffices to show that (using the union bound)
\begin{align}\label{eq:probability_union_bound}
\mathbb{P}\left( \bigcup_{i\in [N]} \hat{G}_i + \alpha \hat{\sigma}_i \leq  G_i \right) < \sum_{i=1}^N \mathbb{P}\left( \hat{G}_i + \alpha \hat{\sigma}_i \leq  G_i \right) < \delta.
\end{align}
To show this, choose an arbitrary $i \in [N]$ and define
\begin{align}
\hat{\sigma}(t+1) &= K\sqrt{NT} + \sum_{\tau=1}^t \frac{1}{p_i(\tau)\sqrt{NT}} \label{eq:sigma_adversarial_definition}\\
s_t &= \frac{\alpha K}{2\hat{\sigma}(t+1)} \leq 1 \label{eq:s_t_adversarial_definition}
\end{align}
Using the shorthand notation $\hat{\sigma}_i := \hat{\sigma}_i(T+1)$, observe
\begin{align*}
&\mathbb{P}\left( \hat{G}_i + \alpha\hat{\sigma}_i \leq G_i \right) \nonumber\\
&\hspace*{0.3cm}= \mathbb{P}\left( \sum_{t=1}^T\left( r_i(t) - \hat{r}_i(t) - \alpha\hat{\sigma}_i/2 \right) \geq \alpha\hat{\sigma}_i/2 \right) \\
&\hspace*{0.3cm}\leq \mathbb{P}\left( s_T \sum_{t=1}^T\left( r_i(t) - \hat{r}_i(t) -\frac{\alpha}{2p_i(t)\sqrt{NT}} \right) \geq \frac{\alpha^2 K}{4}\right) \\
&\hspace*{0.3cm}= \mathbb{P}\left( \exp\left[s_T \sum_{t=1}^T\left( r_i(t) - \hat{r}_i(t) -\frac{\alpha}{2p_i(t)\sqrt{NT}} \right)\right] \right.\\
&\hspace*{0.8cm}\left. \geq \exp\left(\frac{\alpha^2 K}{4}\right)\right) \nonumber \\
&\hspace*{0.3cm}= \exp\left( -\frac{\alpha^2 K}{4} \right) \mathbb{E}\left[ s_T \sum_{t=1}^T\left( r_i(t) - \hat{r}_i(t) -\frac{\alpha}{2p_i(t)\sqrt{NT}} \right) \right].
\end{align*}
As in Lemma 6.1 from \cite{Auer:2002aa}, define
\begin{align*}
Z_t = \exp\left( s_t \sum_{\tau=1}^t \left( r_i(\tau) - \hat{r}_i(\tau) - \frac{\alpha}{2p_i(\tau)\sqrt{NT}} \right) \right)
\end{align*}
from which it follows for $t=2, \ldots, T$ that
\begin{align*}
Z_t = \exp\left( s_t \left( r_i(t) - \hat{r}_i(t) - \frac{\alpha}{2p_i(t)\sqrt{NT}} \right) \right) Z_{t-1}^{\frac{s_t}{s_{t-1}}}.
\end{align*}
Since
\begin{align*}
\frac{\alpha}{2p_i(t)\sqrt{NT}} \geq \frac{\alpha K}{2p_i(t)\hat{\sigma}_i(t+1)} = \frac{s_t}{p_i(t)},
\end{align*}
we obtain for $t=2, \ldots, T$:
\begin{align*}
&\mathbb{E}_{\hat{r}_i(t)}[Z_t] \\
&\hspace*{0.3cm}\leq \mathbb{E}_{\hat{r}_i(t)} \left[ \exp\left[ s_t \left( r_i(t) - \hat{r}_i(t) - \frac{s_t}{p_i(t)} \right) \right] \right] Z_{t-1}^{\frac{s_t}{s_{t-1}}} \nonumber \\
&\hspace*{0.3cm}\leq \mathbb{E}_{\hat{r}_i(t)}\left[ 1 + s_t\left( r_i(t) - \hat{r}_i(t) \right) + s_t^2 \left( r_i(t) - \hat{r}_i(t) \right)^2 \right] \\
&\hspace*{0.7cm}\times \exp\left( - \frac{s_t^2}{p_i(t)} \right) Z_{t-1}^{\frac{s_t}{s_{t-1}}} \\
&\hspace*{0.3cm}\leq \left(1 + \frac{s_t^2}{p_i(t)}\right) \exp\left( - \frac{s_t^2}{p_i(t)} \right)Z_{t-1}^{\frac{s_t}{s_{t-1}}} \nonumber \\
&\hspace*{0.3cm}\leq Z_{t-1}^{\frac{s_t}{s_{t-1}}} \leq 1 + Z_{t-1}.
\end{align*}
Since $\mathbb{E}_{\hat{r}_i(1)}[Z_1] \leq 1$, it follows that $\mathbb{E}_{\hat{r}_i(T)}[Z_T] < T$. Hence, \eqref{eq:probability_union_bound} writes
\begin{align}
&\sum_{i=1}^N \mathbb{P}\left( \hat{G}_i + \alpha \hat{\sigma}_i \leq  G_i \right) \nonumber\\
&\hspace*{0.3cm}\leq \sum_{i=1}^N \exp\left( -K\frac{N-K}{N-1}\log\left(\frac{NT}{\delta}\right) \right) T \nonumber \\
&\hspace*{0.3cm}=NT \left( \frac{\delta}{NT} \right)^{\frac{K(N-K)}{N-1}} \leq NT \frac{\delta}{NT} = \delta. \label{eq:power_of_less_thanone}
\end{align}
In \eqref{eq:power_of_less_thanone}, we used the fact that the minima of $K(N-K)/(N-1)$ for $1\leq K < N$ are attained at $K=1$ and $K=N-1$ and have value 1. Since $\delta/(NT) < 1$, the claim follows.
\end{proof}

\subsubsection*{Proof of Lemma \ref{lem:gain_exp3pm_bound}}
\begin{duplicateLemma}
For $\alpha \leq 2\sqrt{NT}$, the gain of Algorithm \texttt{Exp3.P.M} is bounded below as follows:
\begin{align}
G_{\texttt{Exp3.P.M}} &\geq \left( 1 - \frac{5}{3}\gamma\right) \hat{U}^\ast - \frac{3N}{\gamma} \log(N/K) \nonumber\\
&\hspace*{0.5cm}- 2\alpha^2 - \alpha(1 + K^2) \sqrt{NT},
\end{align}
where $\hat{U}^\ast = \sum_{j\in a^\ast}\hat{G}_j + \alpha \hat{\sigma}_j$ denotes the upper confidence bound of the optimal gain achieved with optimal set $a^\ast$.
\end{duplicateLemma}


\begin{proof}
From the definition of the weights in Algorithm \texttt{Exp3.P.M}, observe:
\begin{align}
&\frac{W_{t+1}}{W_t} = \sum_{i\in [N]\setminus \tilde{S}(t)} \frac{w_i(t)}{W_t}\exp\left( \eta\hat{r}_i(t) + \frac{\alpha\eta}{p_i(t)\sqrt{NT}} \right) \nonumber\\
&\hspace*{1.5cm}+ \sum_{i\in\tilde{S}(t)}\frac{w_i(t)}{W_t} \nonumber \\
&\hspace*{0.3cm}\leq \sum_{i\in [N]\setminus \tilde{S}(t)} \frac{w_i(t)}{W_t} \left[ 1 + \eta\hat{r}_i(t) + \frac{\alpha\eta}{p_i(t)\sqrt{NT}} \right. \nonumber\\
&\hspace*{0.7cm} \left.+ 2\eta^2 \hat{r}_i(t)^2 + \frac{2\alpha^2\eta^2}{p_i(t)^2 NT} \right] + \sum_{i\in\tilde{S}(t)}\frac{w_i(t)}{W_t}\label{eq:adversarial_multiplay_high_prob_i} \\
&\hspace*{0.3cm}= 1 + \frac{W_t^\prime}{W_t}\sum_{i\in [N]\setminus \tilde{S}(t)} \frac{\frac{p_i(t)}{K}-\frac{\gamma}{k}}{1-\gamma} \left[ \eta\hat{r}_i(t) + \frac{\alpha\eta}{p_i(t)\sqrt{NT}} \right.\nonumber\\
&\hspace*{0.7cm}\left. + 2\eta^2 \hat{r}_i(t)^2 + \frac{2\alpha^2\eta^2}{p_i(t)^2 NT} \right] \\
&\hspace*{0.3cm}\leq 1 + \frac{\eta}{K(1-\gamma)}\sum_{i\in [N]\setminus \tilde{S}(t)} p_i(t)\hat{r}_i(t) + \frac{\alpha\eta}{K(1-\gamma)}\sqrt{\frac{N}{T}} \nonumber\\
&\hspace*{0.7cm} + \frac{2\eta^2}{K(1-\gamma)} \sum_{i\in [N]} p_i(t) \hat{r}_i(t)^2 \nonumber\\
&\hspace*{0.7cm}+ \frac{2\alpha^2 \eta^2}{NTK(1-\gamma)}\sum_{i\in [N]}\frac{1}{p_i(t)}\nonumber \\
&\hspace*{0.3cm}= 1 + \frac{\eta}{K(1-\gamma)}\sum_{i\in a_t}r_i(t) + \frac{\alpha\eta}{K(1-\gamma)}\sqrt{\frac{N}{T}} \nonumber\\
&\hspace*{0.7cm}+ \frac{2\eta^2}{K(1-\gamma)} \sum_{i\in [N]} \hat{r}_i(t) + \frac{2\alpha^2 \eta}{K(1-\gamma)}\frac{1}{T}, \label{eq:adversarial_multiplay_high_prob_ii}
\end{align}
where we used the properties
\begin{align*}
\hat{r}_i(t) &\leq \frac{1}{p_i(t)} \leq \frac{N}{\gamma K} \\ 
\sum_{i\in [N]}p_i(t)\hat{r}_i(t) &= \sum_{i\in [N]} r_i(t) \\
\sum_{i\in [N]}p_i(t)\hat{r}_i(t)^2 &\leq \sum_{i\in [N]}\hat{r}_i(t)
\end{align*}
in \eqref{eq:adversarial_multiplay_high_prob_ii} and the inequality $e^x \leq 1 + x + x^2$ valid for $x\leq 1$ in \eqref{eq:adversarial_multiplay_high_prob_i}. Summing over $t=1, \ldots, T$ and utilizing the telescoping property of the logarithm yields
\begin{align}
&\log\left( \frac{W_{T+1}}{W_1} \right) \nonumber\\
&\hspace*{0.3cm}\leq \frac{\eta}{K(1-\gamma)}G_{\texttt{Exp3.P.M}} + \frac{2\eta^2}{K(1-\gamma)} \sum_{t=1}^T \sum_{i\in [N]} \hat{x}_i(t) \nonumber \\
&\hspace*{0.3cm}+ \frac{\alpha\eta\sqrt{NT}}{K(1-\gamma)} + \frac{2\alpha^2 \eta}{K(1-\gamma)} \nonumber \\
&\hspace*{0.3cm}\leq \frac{\eta}{K(1-\gamma)}G_{\texttt{Exp3.P.M}} + \frac{2\eta^2}{K(1-\gamma)}\frac{N}{K} \hat{U}^\ast + \nonumber \\
&\hspace*{0.3cm}\frac{\alpha\eta\sqrt{NT}}{K(1-\gamma)} + \frac{2\alpha^2 \eta}{K(1-\gamma)}. \label{eq:adversarial_multiplay_high_prob_v}
\end{align}
On the other hand, we have
\begin{align}
\log(W_1) &= \log\left[ N \exp\left( \frac{\alpha\gamma K^2}{3} \sqrt{\frac{T}{N}} \right) \right] \nonumber\\
&= \log(N) + \alpha K \eta\sqrt{NT}, \label{eq:adversarial_multiplay_high_prob_iii}\\
\log(W_{T+1}) &\geq \log\left( \sum_{i\in a^\ast} w_j(T+1)\right) \nonumber\\
&\hspace*{-1.3cm}\geq \log\left[ K\left( \prod_{i\in a^\ast} w_j(T+1) \right)^{1/K} \right] \nonumber \\
&\hspace*{-1.3cm}= \log(K) + \frac{1}{K}\sum_{i\in a^\ast} \log(w_i(T+1))\nonumber \\
&\hspace*{-1.3cm}= \log(K) + \frac{1}{K}\sum_{i\in a^\ast} \left[ \frac{\alpha\gamma K^2}{3}\sqrt{\frac{T}{N}} \right.\nonumber\\
&\hspace{-0.83cm}\left. + \sum_{t=1}^T\left(\eta \hat{x}_i(t) + \frac{\alpha\eta}{p_i(t)\sqrt{NT}}\right) \right]\nonumber\\
&\hspace*{-1.3cm}= \log(K) + \frac{1}{K}\sum_{i\in a^\ast}\left( \eta \hat{G}_i + \alpha\eta \hat{\sigma}_i \right), \label{eq:adversarial_multiplay_high_prob_iv}
\end{align}
where \eqref{eq:adversarial_multiplay_high_prob_iii} and
\eqref{eq:adversarial_multiplay_high_prob_iv} follow from the definitions of weights in Algorithm \texttt{Exp3.P.M} and \eqref{eq:sigma_adversarial_definition}, respectively.
\end{proof}

Finally, to show the claim in Theorem \ref{thm:multiple_play_fixed_round_high_probability}, simply combine the results from Lemma \ref{lem:confidence_level_exp3p} and Lemma \ref{lem:gain_exp3pm_bound}. Combining \eqref{eq:adversarial_multiplay_high_prob_v}, \eqref{eq:adversarial_multiplay_high_prob_iii}, and \eqref{eq:adversarial_multiplay_high_prob_iv} yields
\begin{align*}
G_{\texttt{Exp3.P.M}} &\geq \left(1 - \frac{5\gamma}{3}\right)\hat{U}^\ast - 2\alpha^2 - \alpha(1 + K^2) \sqrt{NT} \\
&\hspace*{0.5cm}- \frac{3N}{\gamma}\log(N/K).
\end{align*}
From Lemma \eqref{lem:confidence_level_exp3p}, it follows that $\hat{U}^\ast > G_{\max}$ with probability at least $1-\delta$. Together with the simple fact $G_{\max}\leq KT$, we have that
\begin{align*}
\mathcal{R} &= G_{\max} - G_{\texttt{Exp3.P.M}} \\
&\leq \frac{5}{3}\gamma KT + 2\alpha^2 + \alpha(1+K^2)\sqrt{NT} + \frac{3N}{\gamma} \log\left(\frac{N}{K}\right).
\end{align*}
Choosing
\begin{align}
\gamma &= \min\left(\frac{3}{5}, \frac{3}{\sqrt{5}}\sqrt{\frac{N\log(N/K)}{KT}}\right),\label{eq:gamma_choice_exp3g}\\
\alpha &= 2\sqrt{\frac{N-K}{N-1} \log\left( \frac{NT}{\delta} \right)} \label{eq:alpha_choice_exp3g}
\end{align}
yields \eqref{eq:high_probability_bound_multiple_play_fixed_T}, which is the bound in Theorem \ref{thm:multiple_play_fixed_round_high_probability}. If either $T \geq \frac{N\log(N/K)}{5K}$ (to make $\gamma \leq 3/5$ in \eqref{eq:gamma_choice_exp3g}) or $\delta \geq NT\exp\left(-\frac{NT(N-1)}{N-K} \right)$ (to make $\alpha < 2\sqrt{NT}$ in \eqref{eq:alpha_choice_exp3g}) is not fulfilled, then the bound holds trivially.

\subsection*{Proof of Theorem \ref{thm:multiple_play_budget_high_probability}}
\begin{duplicateTheorem}
For the multiple play algorithm ($1\leq K \leq N$) and the budget $B > 0$, the following bound on the regret holds with probability at least $1-\delta$:
\begin{align}
\mathcal{R} &= G_{\max} - G_{\texttt{Exp3.P.M.B}} \nonumber \\
&\leq 2\sqrt{3}\sqrt{\frac{NB(1-c_{\min})}{c_{\min}}\log\frac{N}{K}} \nonumber\\
&\hspace*{0.4cm}+ 4\sqrt{6}\frac{N-K}{N-1}\log\left( \frac{NB}{Kc_{\min}\delta} \right) \\
&\hspace*{0.4cm}+ 2\sqrt{6}(1+K^2)\sqrt{\frac{N-K}{N-1}\frac{NB}{Kc_{\min}}\log\left( \frac{NB}{Kc_{\min}\delta} \right)}\nonumber
\end{align}
\end{duplicateTheorem}


\subsubsection*{Proof of Lemma \ref{lem:confidence_level_exp3p_budget}}
\begin{duplicateLemma}
For $2\sqrt{6}\sqrt{\frac{N-K}{N-1}\log\frac{NB}{Kc_{\min}\delta}} \leq \alpha \leq 12\sqrt{\frac{NB}{Kc_{\min}}}$,
\begin{align}
&\mathbb{P}\left( \hat{U}^\ast > G_{\max} - B \right) \nonumber \\
&\hspace*{0.26cm}\geq\mathbb{P}\left( \bigcap_{a\subset\mathcal{S}} \sum_{i\in a} \hat{G}_i - \hat{L}_i + \alpha \hat{\sigma}_i > \sum_{i\in a} G_i - L_i \right) \geq 1-\delta,
\end{align}
where $a\subset\mathcal{S}$ denotes an arbitrary time-invariant subset of $1 \leq K < N$ unique elements from $[N]$. $\hat{U}^\ast$ denotes the upper confidence bound for the cumulative optimal gain minus the cumulative cost incurred after $\tau_a(B)$ rounds (the stopping time when the budget is exhausted):
\begin{align}
S^\ast &= \max_{a\in\mathcal{S}}\sum_{t=1}^{\tau_S(B)}(r_i(t) - c_i(t)) \nonumber\\
\hat{U}^\ast &= \sum_{i\in a^\ast}\left(\alpha\hat{\sigma}_i +\sum_{t=1}^{\tau_{a^\ast}(B)}(\hat{r}_i(t) - \hat{c}_i(t))\right)
\end{align}
\end{duplicateLemma}


\begin{proof}
As in the proof for Lemma \ref{lem:confidence_level_exp3p}, it suffices to show that 
\begin{align}\label{eq:probability_union_bound_budget}
&\mathbb{P}\left( \bigcup_{i\in [N]} \hat{G}_i - \hat{L}_i + \alpha \hat{\sigma}_i \leq  G_i - L_i \right) \nonumber\\
&\hspace*{0.3cm}< \sum_{i=1}^N \mathbb{P}\left( \hat{G}_i - \hat{L}_i + \alpha \hat{\sigma}_i \leq  G_i - L_i \right) < \delta.
\end{align}
Let $\hat{\sigma}(t+1)$ and $s_t$ be defined as
\begin{align}
\hat{\sigma}(t+1) &= K\sqrt{\frac{NB}{Kc_{\min}}} + \sum_{\tau=1}^t \frac{\sqrt{Kc_{\min}}}{p_i(\tau)\sqrt{NB}} \label{eq:sigma_adversarial_def_budget}\\
s_t &= \frac{\alpha K}{12\hat{\sigma}(t+1)} \leq 1 \label{eq:s_t_adversarial_def_budget}
\end{align}
Now observe
\begin{align*}
&\mathbb{P}\left( \hat{G}_i - \hat{L}_i + \alpha\hat{\sigma}_i \leq G_i - L_i \right) \nonumber\\
&\hspace*{0.3cm}= \mathbb{P}\left( \sum_{t=1}^{\tau_a(B)}\left( r_i(t) - c_i(t) - \hat{r}_i(t) + \hat{c}_i(t) - \alpha\hat{\sigma}_i/2 \right) \right.\\
&\hspace*{1.6cm}\left.\geq \alpha\hat{\sigma}_i/2 \right) \\
&\hspace*{0.3cm}\leq \mathbb{P}\left( s_{\tau_a(B)} \sum_{t=1}^{\tau_a(B)}\left( r_i(t) - c_i(t) - \hat{r}_i(t) + \hat{c}_i(t)\right.\right.\\
&\hspace*{1.6cm}\left.\left. -\frac{\alpha\sqrt{Kc_{\min}}}{2p_i(t)\sqrt{NB}} \right) \geq \frac{\alpha^2 K}{24}\right) \\
&\hspace*{0.3cm}= \exp\left( -\frac{\alpha^2 K}{24} \right) \mathbb{E}\left[ s_{\tau_a(B)} \sum_{t=1}^{\tau_a(B)}\left( r_i(t) - c_i(t) \right.\right.\\
&\hspace*{1.6cm}\left.\left.- \hat{r}_i(t) + \hat{c}_i(t) -\frac{\alpha\sqrt{Kc_{\min}}}{2p_i(t)\sqrt{NB}} \right) \right].
\end{align*}
Now define $Z_t$ as follows:
\begin{align*}
&Z_t = \exp\left( s_t \sum_{\tau=1}^t \left( r_i(\tau) - c_i(\tau) - \hat{r}_i(\tau) + \hat{c}_i(\tau) \right.\right.\\
&\hspace*{1.6cm}\left.\left.- \frac{\alpha \sqrt{Kc_{\min}}}{2p_i(\tau)\sqrt{NB}} \right) \right)
\end{align*}
from which it follows that
\begin{align*}
&Z_t = \exp\left( s_t \left( r_i(t) - c_i(t) - \hat{r}_i(t) + \hat{c}_i(t) - \right.\right.\\
&\hspace*{1.6cm}\left.\left.\frac{\alpha \sqrt{Kc_{\min}}}{2p_i(t)\sqrt{NB}} \right) \right) Z_{t-1}^{\frac{s_t}{s_{t-1}}},~t=2, \ldots, {\tau_S(B)}.
\end{align*}
Since
\begin{align*}
\frac{\alpha \sqrt{Kc_{\min}}}{2p_i(t)\sqrt{NB}} \geq \frac{4 \alpha K}{8p_i(t)\hat{\sigma}_i(t+1)} = \frac{4s_t}{p_i(t)},
\end{align*}
we obtain for $t=2, \ldots, {\tau_S(B)}$:
\begin{align}
\mathbb{E}_{t}[Z_t] &\leq \mathbb{E}_{t} \left[ \exp\left[ s_t \left( r_i(t) - \hat{r}_i(t) - \frac{4s_t}{p_i(t)} \right) \right] \right] Z_{t-1}^{\frac{s_t}{s_{t-1}}} \nonumber \\
&\leq \mathbb{E}_{t}\left[ 1 + s_t\left( r_i(t) - c_i(t) - \hat{r}_i(t) + \hat{c}_i(t) \right) \right.\nonumber\\
&\hspace*{0.9cm}\left.+ s_t^2 \left( r_i(t) - c_i(t) - \hat{r}_i(t) + \hat{c}_i(t) \right)^2 \right]\nonumber\\
&\quad\times \exp\left( - \frac{4s_t^2}{p_i(t)} \right) Z_{t-1}^{\frac{s_t}{s_{t-1}}} \label{eq:bound_Zt_adversarial_budget}\\
&\leq \left(1 + \frac{4s_t^2}{p_i(t)}\right) \exp\left( - \frac{4s_t^2}{p_i(t)} \right)Z_{t-1}^{\frac{s_t}{s_{t-1}}} \nonumber \\
&\leq Z_{t-1}^{\frac{s_t}{s_{t-1}}} \leq 1 + Z_{t-1}
\end{align}
In \eqref{eq:bound_Zt_adversarial_budget}, we used the following operation:
\begin{align*}
&\mathbb{E}_t\left[((r_i(t) - \hat{r}_i(t)) - (c_i(t) - \hat{c}_i(t))^2\right] \nonumber\\
&\hspace*{0.3cm}=\mathbb{E}_t\left[(r_i(t) - \hat{r}_i(t))^2\right] + \mathbb{E}_t\left[(c_i(t) - \hat{c}_i(t))^2\right] \\
&\hspace*{0.7cm}- 2\mathbb{E}_t\left[(r_i(t) - \hat{r}_i(t))(c_i(t) - \hat{c}_i(t))\right]\\
&\hspace*{0.3cm}\leq \mathbb{E}_t[\hat{r}_i(t)^2] + \mathbb{E}_t[\hat{c}_i(t)^2] \\
&\hspace*{0.7cm}- 2\mathbb{E}_t[r_i(t)c_i(t) - r_i(t)\hat{c}_i(t) - c_i(t)\hat{r}_i(t) + \hat{r}_i(t)\hat{c}_i(t)]\\
&\hspace*{0.3cm}\leq \frac{2}{p_i(t)} - 2\left[r_i(t)c_i(t) - 2 r_i(t)c_i(t) + \frac{r_i(t)c_i(t)}{p_i(t)}\right]\nonumber\\
&\hspace*{0.3cm} \leq \frac{2}{p_i(t)} + 2\frac{1}{p_i(t)} = \frac{4}{p_i(t)}.
\end{align*}

Since $\mathbb{E}_{t}[Z_1] \leq 1$, it follows that $\mathbb{E}_{\tau_a(B)}[Z_{\tau_a(B)}] < \tau_S(B)$. Hence, \eqref{eq:probability_union_bound_budget} writes
\begin{align}
&\sum_{i=1}^N \mathbb{P}\left( \hat{G}_i - \hat{L}_i + \alpha \hat{\sigma}_i \leq  G_i - L_i \right) \nonumber\\
&\hspace*{0.3cm}\leq \sum_{i=1}^N \exp\left( -K\frac{N-K}{N-1}\log\left(\frac{NB}{Kc_{\min}\delta}\right) \right) {\tau_a(B)} \nonumber \\
&\hspace*{0.3cm}=N{\tau_a(B)} \left( \frac{Kc_{\min}\delta}{NB} \right)^{\frac{K(N-K)}{N-1}} \leq N{\tau_a(B)} \frac{Kc_{\min}\delta}{NB} \leq \delta \label{eq:power_of_less_thanone_budget}
\end{align}
because $
\frac{\tau_a(B)}{B/(Kc_{\min})} \leq 1$. This completes the proof.
\end{proof}

\subsubsection*{Proof of Lemma \ref{lem:gain_exp3pm_bound_budget}}
\begin{duplicateLemma}
For $\alpha \leq 2\sqrt{\frac{NB}{Kc_{\min}}}$, the gain of Algorithm \texttt{Exp3.P.M.B} is bounded below as follows:
\begin{align}
G_{\texttt{Exp3.P.M.B}} &\geq \left( 1 - \gamma - \frac{2\gamma}{3}\frac{1-c_{\min}}{c_{\min}}\right)\hat{U}^\ast \\
&\hspace*{0.4cm}- \frac{3N}{\gamma}\log\frac{N}{K} - 2\alpha^2  - \alpha(1 + K^2)\frac{BN}{Kc_{\min}}. \nonumber
\end{align}
\end{duplicateLemma}


\begin{proof}
Using the weight update rule for Algorithm \texttt{Exp3.P.M.B}, we obtain (using the same manipulations as in the proof for Lemma \ref{lem:gain_exp3pm_bound})
\begin{align}
\frac{W_{t+1}}{W_t} &= 1 + \frac{\eta}{K(1-\gamma)}\sum_{t=1}^T \sum_{i\in a_t} (r_i(t) - c_i(t)) \nonumber\\
&\hspace*{0.4cm}+ \frac{\alpha\eta T}{K(1-\gamma)}\sqrt{\frac{NKc_{\min}}{B}} + \frac{2\alpha^2\eta Kc_{\min} T}{BK(1-\gamma)}\nonumber\\
&\quad + \frac{2\eta^2 (1-c_{\min})}{K(1-\gamma)}\sum_{i\in [N]}\sum_{t=1}^T (\hat{r}_i(t) - \hat{c}_i(t)),\label{eq:Wt1_Wt_adversarial}
\end{align}
where $T = \max\left( \tau_{a^\ast}(B), \tau_{a}(B) \right)$. On the other hand, observe that
\begin{align}
\log W_1 &= \log N + \alpha K\eta\sqrt{\frac{BN}{Kc_{\min}}} \label{eq:logW1_adversarial}
\end{align}
and
\begin{align}
&\log W_{T+1} \geq \log K + \frac{1}{K}\sum_{i\in a^\ast} \log w_j(T+1) \nonumber\\
&\hspace*{0.4cm}\geq \log K + \frac{1}{K}\sum_{i\in a^\ast}\left( \alpha\eta K\sqrt{\frac{BN}{Kc_{\min}}} \right.\nonumber\\
&\hspace*{0.8cm}\left. + \sum_{t=1}^T\left( \eta(\hat{r}_i(t) - \hat{c}_i(t)) + \frac{\alpha\eta\sqrt{Kc_{\min}}}{p_i(t)\sqrt{NB}} \right) \right) \nonumber\\
&\hspace*{0.4cm}= \log K + \frac{1}{K}\sum_{i\in a^\ast}\left( \alpha\eta\hat{\sigma}_i(T+1) \right.\nonumber\\
&\hspace*{0.8cm}\left.+ \eta(\hat{G}_i(T+1) - \hat{L}_i(T+1)) \right) \nonumber \\
&\hspace*{0.4cm}\geq \log K + \frac{1}{K}\sum_{i\in a^\ast}\left( \alpha\eta\hat{\sigma}_i(\tau_{a^\ast}(B)+1) \right.\nonumber\\
&\hspace*{0.8cm}\left. + \eta(\hat{G}_i(\tau_{a^\ast}(B)+1) - \hat{L}_i(\tau_{a^\ast}(B)+1)) \right) \nonumber \\
&\hspace*{0.4cm}= \log K + \frac{\eta}{K}\hat{U}^\ast.\label{eq:logWt1_adversarial}
\end{align}
Using the identity $e^x > 1 + x$, the telescoping property of the logarithm in equations \eqref{eq:Wt1_Wt_adversarial}, and \eqref{eq:logW1_adversarial} and \eqref{eq:logWt1_adversarial} yield
\begin{align}
&\log\frac{K}{N} + \frac{\eta}{K}\hat{U}^\ast - \alpha K\eta\sqrt{\frac{BN}{Kc_{\min}}} \nonumber\\
&\hspace*{0.4cm}\leq \frac{\eta}{K(1-\gamma)}\sum_{t=1}^T \sum_{i\in a_t}(r_i(t) - c_i(t)) \nonumber\\
&\hspace*{0.8cm}+ \frac{\alpha\eta T}{K(1-\gamma)}\sqrt{\frac{NKc_{\min}}{B}} + \frac{2\alpha^2\eta K c_{\min}T}{BK(1-\gamma)}\nonumber\\
&\hspace*{0.8cm}+ \frac{2\eta^2 (1-c_{\min})}{K(1-\gamma)}\sum_{i\in [N]}\sum_{t=1}^T (\hat{r}_i(t) - \hat{c}_i(t)) \label{eq:gain_exp3pm_manip}
\end{align}
From Lemma \ref{lem:confidence_level_exp3p_budget}, we have that $\hat{U}^\ast > G_{\max}-B$ with probability at least $1-\delta$. Now, manipulating the right hand side of \eqref{eq:gain_exp3pm_manip} and noticing that algorithm \texttt{Exp3.P.M.B} terminates after $\tau_{\mathcal{A}}(B)$ rounds yields
\begin{align}
\mathrm{RHS} &\leq \frac{\eta}{K(1-\gamma)}\left( G_{\texttt{Exp3.P.M.B}} - (B-Kc_{\max})\right) \nonumber\\
&\hspace*{0.4cm}+ \frac{\alpha\eta}{K(1-\gamma)}\sqrt{\frac{NB}{Kc_{\min}}} + \frac{2\alpha^2\eta}{K(1-\gamma)} \nonumber\\
&\hspace*{0.4cm}+ \frac{2\eta^2 (1-c_{\min})}{K(1-\gamma)}\sum_{t=1}^{B/(Kc_{\min})} (\hat{r}_i(t) - \hat{c}_i(t)) ,
\end{align}
where we used the fact that $T = \max\left( \tau_{a^\ast}(B), \tau_{a}(B) \right) \leq B/(Kc_{\min})$. Finally, putting LHS and RHS together and utilizing $\sum_{t=1}^{B/(Kc_{\min})} (\hat{r}_i(t) - \hat{c}_i(t)) \leq (N/K)\hat{U}^\ast$  gives
\begin{align*}
&\frac{K(1-\gamma)}{\eta}\log\frac{K}{N} + (1-\gamma)\hat{U}^\ast - \alpha K^2(1-\gamma)\sqrt{\frac{BN}{Kc_{\min}}} \nonumber\\
&\hspace*{0.4cm}\leq G_{\texttt{Exp3.P.M.B}} - B + K + \alpha\sqrt{\frac{BN}{Kc_{\min}}} \\
&\hspace*{0.8cm}+ 2\eta(1-c_{\min})\frac{N}{K}\hat{U}^\ast + 2\alpha^2,
\end{align*}
from which \eqref{eq:gain_exp3pm_bound} follows.
\end{proof}

Finally, putting both Lemmas together, we get from Lemma \ref{lem:confidence_level_exp3p_budget} that $\hat{U}^\ast > G_{\max}$ with probability at least $1-\delta$. Also, note that $G_{\max} = B/c_{\min}$. Combining with Lemma \ref{lem:gain_exp3pm_bound_budget} and choosing
\begin{align}
\gamma &= \min\left( \left(1 + \frac{2}{3}\frac{1-c_{\min}}{c_{\min}}\right)^{-1},\right.\label{eq:gamma_choice_exp3g_budget}\\
&\hspace*{0.4cm}\left. \left( \frac{3N\log(N/K)}{(G_{\max}-B)\left( 1 + 2(1-c_{\min})/(3c_{\min}) \right)}, \right)^{1/2} \right), \nonumber\\
\alpha &= 2\sqrt{6}\sqrt{\frac{N-K}{N-1}\log\left( \frac{NB}{Kc_{\min}\delta} \right)}\label{eq:alpha_choice_exp3g_budget}
\end{align}
yields the desired bound \eqref{eq:high_probability_bound_budget}. If either $B \geq 3N\log(N/K)\left(1 + 2/3 + c_{\min}/(1-c_{\min})\right)$ (to make $\gamma \leq 3/5$ in \eqref{eq:gamma_choice_exp3g_budget}) or $\delta \geq NB/(Kc_{\min})\exp\left(-\frac{6(N-1)NB}{(N-K)Kc_{\min}} \right)$ (to make $\alpha < 12\sqrt{NT/(Kc_{\min})}$ in \eqref{eq:alpha_choice_exp3g_budget}) is not fulfilled, then the bound holds trivially.

\end{document}